\documentclass[12pt]{article}

\usepackage[utf8]{inputenc} %
\usepackage[T1]{fontenc}    %
\usepackage{color}
\definecolor{asdf}{rgb}{0.1333,0.5451,0.1333}
\usepackage[colorlinks,linkcolor=red,citecolor=red,urlcolor=red]{hyperref}
\usepackage{url}            %
\usepackage{booktabs}       %
\usepackage{amsfonts}       %
\usepackage{nicefrac}       %
\usepackage{microtype}      %
\usepackage{fullpage}

\usepackage{url}
\usepackage{graphicx}
\usepackage[font=scriptsize]{caption}
\usepackage{amsmath}

\usepackage{amsfonts}
\usepackage{subcaption}
\usepackage{graphicx} %

\usepackage{xcolor}
\usepackage{algorithm,algorithmic}
\usepackage{tabularx} %
\usepackage{amsthm}
\usepackage{amsfonts}
\usepackage{amsmath}
\usepackage{amssymb}
\usepackage{mathrsfs}
\usepackage{multirow}
\usepackage{bbm}
\usepackage[normalem]{ulem}

\usepackage{thmtools}
\usepackage{thm-restate}

\renewcommand{\subset}{\subseteq}

\newtheorem{theorem}{Theorem}
\newtheorem{definition}{Definition}

\newtheorem{remark}{Remark}

\newtheorem{proposition}{Proposition}
\newtheorem{lemma}{Lemma}
\newtheorem{fact}{Fact}
\newcounter{assump}
\newtheorem{assumption}[assump]{Assumption}

\DeclareMathOperator*{\argmax}{arg\,max}
\DeclareMathOperator*{\argmin}{arg\,min}

\DeclareMathOperator*{\subg}{subG}
\DeclareMathOperator*{\sube}{subE}

\usepackage{enumitem}
\usepackage{wrapfig}
\DeclareMathOperator*{\tr}{tr}

\newcommand{\beq}{\begin{equation}}

\newcommand{\iid}{\stackrel{\text{iid}}{\sim}}
\newcommand{\eeq}{\end{equation}}
\newcommand{\beqs}{\begin{equation*}}
\newcommand{\eeqs}{\end{equation*}}

\renewcommand{\AA}{\mathcal{A}}

\newcommand{\OO}{\mathcal{O}}

\newcommand{\II}{\mathcal{I}}
\newcommand{\R}{\mathbb{R}}

\newcommand{\EE}{\mathcal{E}}

\newcommand{\FF}{\mathcal{F}}

\newcommand{\N}{\mathbb{N}}

\newcommand{\NN}{\mathcal{N}}

\newcommand{\regret}{\text{Reg}}
\newcommand{\E}{\mathbb{E}}
\newcommand{\I}{\mathbb{I}}
\newcommand{\CC}{\mathcal{C}}
\newcommand{\XX}{\mathcal{X}}

\newcommand{\DD}{\mathcal{D}}

\newcommand{\<}{\left<}
\renewcommand{\>}{\right>}

\newcommand{\norm}[1]{\left\lVert#1\right\rVert}

\newcommand{\abs}[1]{\ensuremath{| #1 |}}

\newcommand{\citet}{\cite}
\newcommand{\citep}{\cite}

\title{Model Selection in Batch Policy Optimization}
\author{Jonathan N.\ Lee$^{\sharp \dagger}$ ~~~ George Tucker$^\dagger$ ~~~ Ofir Nachum$^\dagger$ ~~~ Bo Dai$^\dagger$ \\ 
\vspace{0.0625in}
\\$^\sharp$Stanford University
\\$^\dagger$Google Research, Brain Team}
\date{}

\begin{document}

\maketitle

\begin{abstract}
  \noindent
  We study the problem of model selection in batch policy optimization: given a fixed, partial-feedback dataset and $M$ model classes, learn a policy with performance that is competitive with the policy derived from the best model class. 
  We formalize the problem in the contextual bandit setting with linear model classes by identifying three sources of error that any model selection algorithm should optimally trade-off in order to be competitive: (1) approximation error,  (2) statistical complexity, and (3) coverage. 
  The first two sources are common in model selection for supervised learning, where optimally trading-off these properties is well-studied. 
  In contrast, the third source is unique to batch policy optimization and is due to dataset shift inherent to the setting. 
  We first show that no batch policy optimization algorithm can achieve a guarantee addressing all three simultaneously, revealing a stark contrast between difficulties in batch policy optimization and the positive results available in supervised learning. 
  Despite this negative result, %
  we show that relaxing any one of the three error sources enables the design of algorithms achieving near-oracle inequalities for the remaining two.  
  We conclude with experiments demonstrating the efficacy of these algorithms. 
\end{abstract}

\pagebreak

\pagebreak

\section{Introduction}

Model selection and hyperparameter tuning are fundamental tasks in supervised learning and statistical learning theory. In these settings, given a model class, the standard goal is to minimize risk, which can always be decomposed as a sum of approximation error (i.e., bias) and estimation error (i.e., variance) of the model class. A trade-off between these two often exists: a large model class may require a lot of data for generalization while a small one may have suffer from large approximation error.  At its core, model selection describes the problem of choosing a model class so as to automatically balance these quantities. A vast literature exists on algorithms and selection rules -- such as hold-out methods, cross-validation, and structural risk minimization -- that nearly optimally achieve this trade-off in supervised learning \citep{massart2007concentration, lugosi1999adaptive,bartlett2002model,bartlett2008fast}. That is, one can select a model class from a collection that nearly matches the performance of the best model class. The implications in practice have been equally, if not more, impactful as evidenced by the widespread use of model validation and selection in machine learning applications, where methods like cross-validation on a held-out dataset are standard and essential steps for practitioners.  

In recent years, interest has turned to model selection in bandits and reinforcement learning \citep{agarwal2017corralling,foster2019model,pacchiano2020model,lee2021online,modi2020sample}.
 However, in contrast to the extensive understanding of model selection in supervised learning and the growing literature in online learning, relatively little is known about model selection in the context of \emph{batch} (or offline) bandits and reinforcement learning. Batch policy optimization, or offline  policy optimization, is a promising paradigm for learning decision-making policies by leveraging large datasets of interactions with the environment~\citep{lange2012batch,levine2020offline}. The goal is for the learner to find a good policy without interacting with the environment in an online fashion so as to avoid dangerous or costly deployments of sub-optimal policies.  While a number of works provide theoretically sample-efficient algorithms for batch policy optimization~\citep{munos2008finite,jin2019provably,nachum2019algaedice,liu2020provably,xie2021bellman}, their effectiveness in practice has been limited due to the lack of tools for validation and selection when faced with multiple options for model classes or hyperparameter settings. %

It is clear that a need exists in batch policy optimization for an analogue to methods like cross-validation in supervised learning. Traditionally, this need has motivated a large body of literature dedicated to the problem of batch policy \textit{evaluation} (e.g., \citep{precup2000eligibility,jiang2016doubly,nachum2019dualdice}). 
These works aim to estimate the values (i.e., online performance) of a set of arbitrary candidate policies, typically via value function regression on the fixed batch dataset or some form of importance sampling. %
Unfortunately, in practice these methods warrant their own hyperparameter selection (e.g. choice of model class used to estimate value functions), an observation that has been noted by several authors \citep{tang2021model,kumar2021workflow,paine2020hyperparameter}. The policy evaluator could thus suffer from the same issues as batch policy optimization algorithms. %
The question of how to achieve effective model selection methods \textit{in the batch setting}
thus remains open.

Motivated by this challenge, in this paper we formalize and study theoretically the problem of model selection for batch policy optimization in the setting of contextual bandits with linear models and make progress towards understanding what is and is not possible in the batch setting compared to supervised learning. The problem is as follows. The learner is given access to a collection of model classes $\FF_1, \ldots \FF_M$ in order to estimate value functions. 
Equipped with a single model class $\FF_k$, a base algorithm produces a policy $\hat \pi_k$. The learner's goal is thus to leverage all $M$ classes to produce a policy $\hat \pi$ that nearly matches the performance of the best $\hat \pi_k$. That is, the learner should perform as if the ``optimal'' model class $\FF_{k_*}$ that produces the best $\hat \pi_{k_*}$ were known in advance.
To ground our results, we focus on the contextual bandit setting 
and model classes $\FF_k$ that are linear with respect to a collection of known feature maps. 

\subsection{Contributions}
Our contributions are as follows:

\paragraph{Three Model Selection Criteria} In Section~\ref{sec::msbatch}, we identify three sources of error that a model selection algorithm should trade-off in order to be competitive with $\hat \pi_{k_*}$ for linear model classes. Two natural sources, borrowed from supervised learning bounds, are \textbf{approximation error}  and statistical \textbf{complexity} of the model class.
However, unlike supervised learning, a third source that contributes to error is the \textbf{coverage} properties of the model class in conjunction with the fixed dataset, due primarily to dataset shift. The fixed dataset may not sufficiently cover the relevant states and actions\footnote{In contextual bandits, one need only worry about action distribution mismatch.} and some model classes may be better equipped to handle this. Finally, we aim to ensure that model selection preserves the property that the learned policy competes well against any well-covered comparator policy \citep{jin2021pessimism,zanette2021provable,xie2021batch}.

\paragraph{Hardness of Model Selection} Our first technical contribution (Theorem~\ref{thm::lower-bound}) shows that it is provably impossible to perform model selection so as to optimally trade-off all three error sources. This is perhaps surprising for two reasons. Firstly, in supervised learning and general risk minimization, such oracle inequalities that nearly optimally balance approximation error and statistical complexity \textit{are} achievable through myriad procedures, and a vast literature exists on this topic.
Secondly, recent work by \cite{su2020adaptive} shows that the analogous problem of estimator selection for batch policy \textit{evaluation} is possible up to an inexact oracle inequality. These observations suggest that the difficulty of model selection is a unique characteristic of the batch policy \textit{selection/optimization} problem.

\paragraph{Positive Results} Despite this negative result, we show in Section~\ref{sec::positives} that positive results for model selection are possible in some instances. Namely, as long as just one of the three error sources is ignored, there exist algorithms to achieve an oracle inequality that optimally trades-off the remaining two. We finally provide experimental results demonstrating the effectiveness of these algorithms.

\subsection{Related Work}

\paragraph{Pessimistic Policy Optimization}
The principle of pessimism in batch policy optimization has recently attracted great interest as both a heuristic and principled method of ensuring that a learned policy stays close to or performs provably well on the given data distribution in order combat the dataset shift problem \citep{xiao2021optimality,jin2021pessimism}.

On the theoretical side, work by  \citet{jin2021pessimism,xiao2021optimality,xie2021bellman,zanette2021provable,uehara2021pessimistic} has sought to quantify the benefit of pessimistic methods for batch learning with different coverage conditions. In particular, \cite{jin2021pessimism} show it is possible to recover regret bounds that are stronger when compared to policies that are well-covered by the data, a property referred to as an ``oracle property.'' This often allows one to overcome worst-case concentrability coefficients that plague other batch methods \citep{chen2019information}.

However, these theoretical studies on the matter of pessimism typically assume access to one single model class and often do not extensively address the problem of approximation error. For example, \citet{jin2021pessimism} and \citet{uehara2021pessimistic} assume realizability and  \citet{xie2021bellman} require that the approximation error be known in advance, which is often not the case. The goal of our paper is to investigate and make progress on these unaddressed issues when multiple model classes are available with unknown approximation error.

\paragraph{Policy Selection}
Another related line of work has considered the problem of batch policy evaluation where one seeks to estimate the value of a target policy using a fixed dataset \citep{duan2020minimax}. Often, the end goal is to evaluate a number of policies and select the one with maximal value~\citep{yang2020offline}. In principle, one could use such an evaluation method to select the estimated best policy from several candidates, which are generated from some batch policy optimization method using different model classes. However, \citet{tang2021model} observe that these evaluation methods are subject to their own modeling errors which may compromise the final regret bound. For example, to ensure accurate policy evaluation, one might be tempted to use a large model class, but the resulting estimation error can compromise the final guarantee.
\cite{tang2021model,kumar2021workflow, paine2020hyperparameter,zhang2021towards} have attempted to address this problem with practical solutions, but theoretical guarantees
are less understood. \cite{farahmand2011model} studied the problem of selecting action-value functions for reinforcement learning, but also assumed access to an estimator, which may suffer from the aforementioned problems. \cite{xie2021batch}  also considered selecting action-value functions  for RL as an application of their algorithm, but did not consider the end-to-end model selection problem with coverage and the resulting guarantee is not competitive with an  oracle. Representation learning \cite{agarwal2020flambe,papini2021leveraging,zhang2021provably} is another related area to policy and model selection, but most current work is in the online setting with assumed realizability and fixed statistical complexity.

Most similar in setting to our work is that of \cite{su2020adaptive} who show that it is possible to select from a collection of batch policy \emph{evaluators} in a way that optimally trades-off approximation error and estimation error up to constants, as long as the estimators are properly ordered. However, they did not consider the \emph{selection} problem, where the best policy will be selected among the candidates that are evaluated upon given data.  
We show in Section~\ref{sec::lower-bound} that this additional task raises complications since different target policies may yield different coverage properties, but these complications can be overcome if the problem is slightly relaxed.

\section{Preliminaries}

\paragraph{Notation} For $n \in \N$, we use $[n]$ to denote the set ${1 ,\ldots, n}$. Let $S \in \R^{d \times d}$ be positive semi-definite and $x \in \R^d$ be a vector. By default, $\| x \|$ denotes the $\ell_2$-norm of $x$ and $\|S \|$ denotes the spectral norm of $S$. $\| x\|_{S} = \sqrt{x^\top S x}$ is the Mahalanobis norm. $\| \cdot \|_{\psi_2}$ and $\| \cdot \|_{\psi_1}$ denote the sub-Gaussian and sub-exponential norms, respectively (see Appendix~\ref{sec::subg} for precise definitions). We use $C, C_1, C_2, \ldots$ to denote absolute constants that do not depend on problem-relevant parameters. We use $\delta > 0$ to denote a desired failure probability and assume $\delta \leq 1/e$.  We write $a \lesssim b$ to mean there exists an absolute constant $C > 0$ such that $a \leq C b$.

\subsection{Contextual Bandits}

We consider the contextual bandit setting \citep{lattimore2018bandit} with state space $\XX$, action space $\AA$, and a fixed distribution $\DD$ over $\XX \times \R^\AA$. A learner interacts with the environment through the following protocol: the environment samples a pair $(X, Y) \sim \DD$ with $X \in \XX$ and $Y \in \R^{\AA}$. The learner observes $X$ and then commits to an action $a \in \AA$. A reward $Y(a)$ is incurred and subsequently only the value of $Y(a)$ is revealed to the learner. 
The learner's objective is to determine a policy $\hat \pi: \XX \to \Delta_\AA$, which maps states to distributions over actions, such that the expected reward $\E_{\hat \pi} Y(A)$ is large. Here $\E_{\hat \pi}$ denotes the expectation over state-action-rewards induced by $\hat \pi$. Similarly, we use $P_{\hat \pi}$ to denote the probability measure under $\hat \pi$. Given a state $x$ and action $a$, we write the expected reward function as $f(x,a) = \E \left[ Y(a) \  | \ x \right]$. We assume $f(x, a) \in [-1, 1]$ and the noise terms $\eta(a) = Y(a) - f(X, a)$ are independent across actions and sub-Gaussian with $\| \eta(a) \|_{\psi_2} \leq 1$. We use $\E_X$ and $P_X$ to denote the expectation and measure over just the marginal distribution over states $X$.

\subsection{Batch Learning}\label{sec::batch-prelim}

As mentioned before, we consider the batch setting for policy optimization~\citep{lange2012batch,levine2020offline,xiao2021optimality}. Rather than learning via direct interaction with the environment, the learner is given access to a fixed dataset $D = \{ x_i, a_i,y_i\}_{i \in [n]}$ of $n \in \N$ prior interactions where $(x_i, \bar y_i) \iid \DD$ and $y_i = \bar y_i(a_i)$ as in the aforementioned interface. Similar to the potential outcomes framework under unconfoundedness~\citep{lin2013agnostic,imbens2015causal}, we assume that $a_i \perp (y_{j}(a))_{j, a} \ |\  x_i$. Intuitively, this ensures the process that selects actions does not peek at the outcomes directly or through a confounding variable. For example, actions could be generated from a fixed  behavior policy, as typically assumed in batch reinforcement learning~\citep{levine2020offline}.
From this data, the learner produces a policy $\hat \pi$ to minimize regret with respect to a comparator policy $\pi$\footnote{For deterministic policies, we abuse notation and write $\pi(x) \in \AA$ to denote the action on which all the probability mass lies given the state $x\in \XX$. We assume all comparator and learned policies are deterministic.}:
\begin{align*}
    \regret(\pi, \hat \pi) & = \E_X  \left[ f(X, \pi(X)) -  f(X, \hat \pi(X)) \right].
\end{align*}
The comparator $\pi$ can be any deterministic policy including the optimal policy.
The typical regret is obtained by making this substitution: $\regret(\hat \pi):= \max_\pi \regret(\pi, \hat \pi)$. Like recent work \cite{jin2021pessimism,zanette2021provable,uehara2021pessimistic,xie2021bellman}, the motivation for this flexibility is that the globally optimal policy may not be well-covered in the dataset, and in these cases we may be more interested in proving regret bounds that compete well against comparators that are well-covered.

In the batch setting, it is often assumed that the actions in $D$ are generated by running a (potentially unknown) fixed stochastic behavior policy $\mu$.
Thus, of central importance is the notion of intrinsic coverage of the action-space offered by $\mu$. A learned policy $\hat \pi$ may induce a distribution over actions that is potentially different from the one induced by $\mu$. As a result, applications of standard supervised learning methods can suffer from sub-optimality due to dataset shift. It is thus common in the literature to assume that the data collection policy $\mu$ sufficiently covers the state-action space, often by means of a concentrability coefficient which captures the worst-case ratio of the density under an arbitrary policy $\pi$ and the data collection policy $\mu$~\citep{munos2008finite}.
\begin{definition}\label{def::concentrability}
	The concentrability coefficient with respect to data collection policy $\mu$ is defined as $\CC(\mu) := \sup_{\pi, x, a} \pi(a | x) / \mu(a | x)$.
\end{definition}
It should be noted that the assumption that such concentrability coefficients are small can be very strong. When using function approximation, it may not be necessary to require that $\mu$ have non-zero density everywhere as long as one can still sample-efficiently find a good fit on the given data distribution. Exploiting the properties of function approximation can thus overcome many coverage-related issues.

\section{Model Selection for Batch Linear Bandit}
\label{sec::msbatch}
	\begin{algorithm}[tb]
		\caption{ Pessimistic Linear Learner } \label{alg::linear}
		\begin{algorithmic}[1]
			\STATE \textbf{Input}: Dataset $D$, linear model class $\FF$, confidence parameter $0 < \delta \leq 1/e$, regularization parameter $\lambda > 0$.
			\STATE Set $V \leftarrow { \lambda \over n } \I_d + {1 \over n } \sum_{i \in [n]} \phi(x_i, a_i)\phi(x_i, a_i)^\top$
			\STATE Set $\hat\theta \leftarrow V^{-1} \left(  {1 \over n} \sum_{i \in [n]} \phi(x_i, a_i) y_i\right)$
			\FOR{$x \in \XX$}
			\STATE $\hat f(x, a) \leftarrow \<\phi(x, a), \hat \theta\>  - \beta_{\lambda, \delta}(n, d) \cdot \| \phi(x, a) \|_{V^{-1}}$
			\STATE Set $\hat \pi(x) \leftarrow \argmax_{a\in\AA} \hat f(x, a)$
			with ties broken arbitrarily
			\ENDFOR
			\STATE \textbf{Return}: $\hat \pi$
		\end{algorithmic}
	\end{algorithm}

In this section, we introduce \textit{linear} model classes for the contextual bandit problem \citep{chu2011contextual,abbasi2011improved,jin2019provably,jin2021pessimism} and present a corresponding batch regret bound for a single model class. We identify three sources of sub-optimality in a resulting regret bound and then formally state the goal of model selection to balance these three sources.

\subsection{Batch Regret for a Single Model Class} %

Let $\FF \subset \left( \XX \times \AA \to \R \right)$ be a model class defined by a \emph{known} $d$-dimensional feature mapping $\phi: \XX \times \AA \to \R^d$ such that
\begin{align}
\FF := \left\{ (x, a) \mapsto \<\phi(x, a), \theta\> \ : \ \theta \in \R^d \right\}.
\end{align}

Leveraging the batch dataset $D$ collected by the policy $\mu$, a natural approach is to simply estimate $f$ from data, e.g. via a least-squares approach, and then extract a policy based on the estimate. This procedure is outlined in Algorithm~\ref{alg::linear}, which performs ridge regression with regularizer $\lambda/n$ on the rewards observed in $D$ and returns a policy $\hat\pi$ that conservatively chooses actions based on both the best-fit estimator and a penalty term determined by the coverage. The penalty term is modulated by a coefficient $\beta_{\lambda,\delta}(n, d)$ that we set to be
\begin{align}
\beta_{\lambda,\delta}(n, d) := \sqrt{ \frac{ \lambda d  } {n}}    + \sqrt{ \frac{ 5 d + 10 {d^{1/2} \log^{1/2}(1/\delta)} + 10 \log(1/\delta) }{n} }.
\end{align}

We note that in general $\FF$ may not actually contain the optimal regressor $f$ that defines the true model. In such cases, we say that $\FF$ may suffer from misspecification or approximation error. It is thus important to ensure that the sub-optimality of the extracted policy scales gracefully with the approximation error of the model, even when the approximation error is not known. The below result, which to the best of our knowledge has not been explicitly stated in the literature, follows as a simple application of a standard regression analysis (e.g. \citet[Section 3.1]{hsu2012random})  to handle concentration and the penalized action-selection method akin to that of \citet{jin2021pessimism}.

\begin{restatable}{theorem}{thmLinear}
\label{thm::linear}
	Let $\hat \pi$ be the output policy of Algorithm~\ref{alg::linear} with $\lambda > 0$. Define,
	\begin{align*}
	    \epsilon(\pi, \hat \pi  ) & =  \E_X \left[  \abs{ f(X,\pi(X)) - \< \phi(X, \pi(X)), \theta_* \> } \right] \\
	&\quad + \E_X\left[ \abs{ \<\phi(X, \hat \pi(X)), \theta_*\> - f(X, \hat \pi(X)) } \right]
	\end{align*}
	where
	    $
	    \theta_* \in \argmin_{\theta \in \R^d}  \sum_{i \in [n]}  \left(  \phi(x_i, a_i)^\top \theta - f(x_i, a_i) \right)^2.
	    $
	If $\| \theta_* \| \leq \sqrt{d}$,
	 then with probability at least $1 - \delta$, for any policy $\pi$ (including the optimal policy), $\regret(\pi, \hat \pi)$ is bounded above by
	\begin{align*}
	 \underbrace{\epsilon(\pi, \hat \pi)}_{\text{approx. error}}  + \:  \underbrace{2\beta_{\lambda, \delta}(n, d)}_{\text{complexity}} \: \cdot \: \underbrace{\E_X \| \phi(X, \pi(X)) \|_{V^{-1}}}_{\text{coverage}}   = \widetilde {\OO} \left( \epsilon(\pi, \hat \pi) + \sqrt{{ d \over n } } \cdot  \E_X\| \phi(X, \pi(X)) \|_{V^{-1}} \right) \label{eq::error-def}
	\end{align*}
	where $V$ is the regularized empirical covariance matrix of the data, defined in Algorithm~\ref{alg::linear}\footnote{Throughout the paper, we use $\widetilde \OO$ to omit polylog factors of problem-dependent parameters such as $\delta^{-1}$, dimension $d$, and number of model classes $M$ (defined in Section~\ref{sec::ms-objectives}). }.
\end{restatable}

The theorem reveals that sub-optimality in the regret bound is due primarily to three sources:
\begin{enumerate}[leftmargin=*,topsep=0pt,parsep=1pt,partopsep=1pt]
    \item \textbf{Approximation error}: this represents how far the closest function in $\FF$ is from representing the true reward function $f$. Here, ``closest'' means the solution, $\theta_*$, to the fixed design regression problem. When $f(x, a) = \phi(x, a)^\top \theta_*$, we say the data is realizable and clearly $\epsilon(\pi, \hat \pi) = 0$. We discuss approximation error and its various forms in more detail in Appendix~\ref{sec::linear-proof}.
    \item \textbf{Statistical complexity}: this represents the learnability size of $\FF$. In the linear case, this is concisely encapsulated by the dimension $d$ of the feature map $\phi$. Note that from the definition of $\beta$, we have $\beta_{\lambda, \delta}(n, d) = \widetilde {\OO} \left(\sqrt{d/n}\right)$
    \item \textbf{Coverage properties}: this source plays an important role in the batch learning setting where sub-optimality can be due to the dataset shift between $D$ and the learned policy $\hat \pi$. In the linear setting, this is represented as the mismatch between the covariance matrix $V$ and the feature distribution of the comparator policy $\pi$. That is, if directions of features that $\pi$ visits do not coincide with directions covered in $V$, then one should expect the error due to mismatch to be large. One critical factor due to pessimism is that we need only consider the mismatch with the comparator $\pi$ (e.g. the optimal policy) as opposed to a worst-case policy or the maximum eigenvalue of $V^{-1}$ or a concentrability coefficient, all of which could lead to a substantially larger regret bound. This ensures that $\hat \pi$ is competitive against all well-covered comparator policies $\pi$ under the dataset $D$ \cite{jin2021pessimism}. 
\end{enumerate}
The precondition of bounded norm of $\theta_*$ is common for both bandits and RL \cite{abbasi2011improved,jin2019provably}. We note, however, that $\theta_*$ depends on the states and actions in the dataset $D$. Several additional remarks are in order.

\begin{remark}
The bound in Theorem~\ref{thm::linear} can be viewed as a data-dependent bound (dependent on the state-action pairs in the training dataset $D$), which is desirable for two main reasons. The first is that it is consistent and easily comparable with prior work \cite{jin2021pessimism,zanette2021provable}. The approximation error $\epsilon(\pi, \hat \pi)$ can be naturally viewed as all of the sub-optimality not accounted for in the usual estimation error, i.e. it recovers the prior work in the realizable case. The second reason is that, as we will see, the data-dependent bound is convenient for model selection since we can evaluate it directly and it avoids any distributional assumptions and quantities until absolutely necessary \cite{duan2020minimax,xie2021bellman}.

In standard statistical learning settings, data-dependent generalization bounds are desirable precisely for the purpose of model selection. They also tend to be tighter than worst-case counterparts \cite{antos2002data,bartlett2002model}.
\end{remark}

\begin{remark}
In general, one does not know the approximation error $\epsilon(\pi, \hat \pi)$ or even non-trivial upper bounds on it, which is the initial motivation for model selection both here and in supervised learning. We work with this particular quantity since it is naturally one of the tightest, and alternatives can be easily derived with some work. See Appendix~\ref{sec::linear-proof} for further discussion.
\end{remark}

\subsection{Model Selection Objectives}\label{sec::ms-objectives}

With these sources of error in mind, we now introduce the general model selection problem for batch policy optimization. 
We assume a collection of $M$ linear model classes $\FF_1, \ldots, \FF_M$, such that, for $k \in [M]$, $\FF_k = \left\{  (x, a) \mapsto \< \phi_k(x, a), \theta\> \ : \ \theta \in \R^{d_k} \right\}$ where $\phi_k$ is a known $d_k$-dimensional feature map for model class $\FF_k$. We desire an algorithm with the following guarantee: given an input dataset $D$ of $n$ interactions and model classes $\FF_1, \ldots, \FF_M$, the algorithm outputs a policy $\hat \pi$ such that, with probability at least $1 - \delta$,
\begin{equation}
\small{
\begin{aligned}\label{eq::oracle-inequality}
\textstyle
& \regret(\pi, \hat \pi)  
& \leq \widetilde \OO  \left( \min_{k \in [M]}  \left\{  \epsilon_k(\pi, \hat \pi) + \sqrt{\frac{d_k}{n}}  \cdot \E_X \| \phi_k(X, \pi(X)) \|_{V^{-1}_k} \right\} \right)
\end{aligned}
}
\end{equation}
for all deterministic policies $\pi$.
Here, $\epsilon_k$ and $V_k$ are the corresponding approximation error and regularized empirical covariance matrix for class $k$, as defined in the previous section for the single model class.

\paragraph{Interpretation} In words, the main goal of model selection is to achieve performance that is nearly as good as the performance that could be achieved had the optimal model class been known in advance. Observe that this desired bound is essentially the best  single model class guarantee from Theorem~\ref{thm::linear} applied to each of the $M$ model classes. 
Such an inequality is often referred to as an \textit{oracle inequality} because an oracle with knowledge of the best class could simply choose it.
We emphasize that achieving the desired bound in \eqref{eq::oracle-inequality} requires careful balancing of all three error sources (approximation, complexity, coverage).  Importantly, note that we aim to maintain the property that $\hat \pi$ is competitive against any well-covered comparator $\pi$. This stands in stark contrast to oracle inequalities in supervised learning, which typically require only balancing approximation error and statistical complexity.

\section{A Negative Result for Model Selection}
\label{sec::lower-bound}

With the main goal of model selection in the batch problem having been introduced, we present our first major contribution, which establishes a fundamental hardness of the model selection problem in the batch policy optimization setting. In particular, we show that, unlike standard learning problems, it is actually \emph{impossible} to optimally trade-off all three error sources that comprise the oracle inequality in \eqref{eq::oracle-inequality}.

Before arriving at the theorem, we identify a condition to impose additional structure on the model classes $\FF_1, \ldots, \FF_M$ so as to actually make model selection \textit{easier}. Otherwise, without structure, the selection problem is trivially impossible. In particular, we consider the setting where the model classes are \textit{nested.}
\begin{definition}\label{def::nested}
The collection of linear model classes $\FF_1, \ldots, \FF_M$ with respective feature maps $\phi_1, \ldots, \phi_M$ is said to be nested if, for each map $\phi_{k +1}$, the first $d_k$ coordinates of $\phi_{k + 1}$ are the same as $\phi_k$ for all $k \in [M - 1]$.
\end{definition}
The nestedness condition imposes structure sufficient for adaptive policy value \textit{estimation} via a variant of the algorithm by \citet{su2020adaptive}. Nestedness effectively requires that the model classes are ordered by complexity. Nonetheless, for policy \textit{optimization}, we will see the additional structure is still insufficient.

We denote $A$ as a model selection algorithm which takes as input the nested model classes $\FF_1, \ldots, \FF_M$ and a dataset $D$ of $n$ interactions and outputs a learned policy $\hat \pi = A(\FF_{1:M}, D)$. 
The following theorem states that even for such nested model classes, near-optimal model selection is impossible, and, in fact, performance can be arbitrarily worse.
\begin{restatable}{theorem}{thmLowerBound}
\label{thm::lower-bound}
	Let $\FF_1, \ldots, \FF_M$ be a particular nested collection of linear model classes in the sense of Definition~\ref{def::nested}. For any $\alpha > 0$, there is $n = \Omega(\alpha^2)$  such that for any algorithm $A$ there is a contextual bandit instance with comparator $\pi$ and dataset $D$ with $n$ interactions consistent with $\DD$ that satisfies
	\begin{align*}
	\frac{ \E_D \left[  \regret(\pi, A(\FF_{1:M}, D))  \right]}{\min_{k} \left\{   \epsilon_k(\pi, \hat \pi) + \sqrt{\frac{d_k}{n}} \cdot \E_X \| \phi_k(X, \pi(X)) \|_{V_k^{-1}}\right\} } \geq \alpha.
	\end{align*}
\end{restatable}
Here, $\E_D$ denotes the expectation with respect to the randomness of the observed rewards in dataset that is distributed according to the contextual bandit instance $\DD$ (in contrast to a high probability  regret bound).
The result follows by reducing the problem to a batch multi-armed bandit problem and designing the appropriate nested model classes that ensure oracle inequality is much better than what is achievable by any algorithm on the bandit problem. The dataset is constructed in order to be sufficiently imbalanced. The construction is illustrative of the core problems that lead to the oracle achieving very small regret. We remark that the result does not necessarily apply to all input collections of model classes; it relies on existence of such a nested collection. Please refer to Appendix~\ref{appendix::lower-bound} for a detailed proof.

The theorem shows that in general, the oracle (ecapsulated by the denominator), which picks the best regret bound induced by the model classes from Theorem~\ref{thm::linear}, can be made to achieve regret that is arbitrarily better than the regret of any model selection algorithm for a sufficiently imbalanced dataset.
The result suggests that the desired bound of \eqref{eq::oracle-inequality} is too ambitious and it highlights a separation of difficulty between model selection in the batch policy optimization setting (where there is an additional error -- coverage -- involved in the oracle inequality) and standard statistical learning.

\section{Positive Results for Special Cases}
\label{sec::positives}

In the last section, we showed that attempting to optimally trade-off all three error sources -- (1) approximation error, (2) complexity, and (3) the coverage property -- is not possible in general for any model selection algorithm. In this section, we explore relaxations of the model selection objective that, rather than requiring all three three to be addressed, consider only certain pairs. 

\subsection{Balancing complexity and coverage}
 
We consider the problem of minimizing regret when the approximation error is zero or we are willing to ignore its contribution to the regret. Such cases might occur, for example, when we have multiple feature representations $\{ \phi_k\}$ that all satisfy realizability, but some may handle coverage better or induce more favorable distributions under the behavior policy $\mu$. 
 
Algorithm~\ref{alg::est-cov} displays a simple selection rule for this case. The main idea is that we will use each model class $\FF_k$ to generate an estimate $\hat \theta_k$ and covariance matrix $V_k$. Then, when extracting the policy, actions are chosen pessimistically, but the action that achieves the highest pessimistic value among all the classes is selected. In a sense, the algorithm is pessimistic regarding values but optimistic across model classes.

 \begin{algorithm}[t]
 	\caption{ Complexity-Coverage Selection } \label{alg::est-cov}
 	\begin{algorithmic}[1]
 		\STATE \textbf{Input}: Dataset $D$, linear model classes $\FF_1, \ldots, \FF_M$, confidence parameter $\delta > 0$, regularization parameter $\lambda > 0$.
 		\FOR{$k \in [M]$}
 			\STATE $V_k \leftarrow { \lambda \over n } \I_d + {1 \over n } \sum_{i \in [n]} \phi_k(x_i, a_i)\phi_k(x_i, a_i)^\top$.
 			\STATE $\hat\theta_k \leftarrow V_k^{-1} \left(  {1 \over n} \sum_{i \in [n]} \phi_k(x_i, a_i) y_i \right)$
 		\ENDFOR
 		
 		\FOR{$x \in \XX$}
 			\STATE  \small{
 			$\hat f_k(x, a) \leftarrow \<\phi_k(x, a), \hat \theta_k\>   - \beta_{\lambda, \delta}(n, d_k)~\cdot~\| \phi_k(x, a) \|_{V_k^{-1}}$
 			}
            \STATE Set $\hat \pi(x), \hat k(x) \leftarrow \argmax_{a \in \AA, k \in [M]} \hat f_k(x, a)$
 			with ties broken arbitrarily
 		\ENDFOR
 		\STATE \textbf{Return}: $\hat \pi$
 	\end{algorithmic}
 \end{algorithm}

The following theorem shows that this procedure optimally trades-off complexity and coverage properties.
\begin{restatable}{theorem}{thmCovEst}
\label{thm::cov-est}
 	Given arbitrary linear model classes $\FF_1, \ldots, \FF_M$, Algorithm~\ref{alg::est-cov} outputs a policy $\hat \pi$ such that, with probability at least $1 - \delta$, for any comparator policy $\pi$, the regret $\regret(\pi, \hat \pi)$ is bounded above by
 	\begin{align*}& \min_{k \in [M]} \left\{   2\beta_{\lambda, \delta/M}(n, d_k) \cdot \E_X \| \phi_k(X, \pi(X)) \|_{V_k^{-1}}\right\} 
 	+  2 \sum_{k \in [M]}\epsilon_{k}(\pi, \hat \pi) 
 	\end{align*}
\end{restatable}
The detailed proof is shown in Appendix~\ref{appendix::cov-est}.
In the case where realizability holds for all of the model classes ($\epsilon_k = 0$ for all $k$), an exact oracle inequality is achieved and we have, with high probability, 
\begin{align*}
    \regret(\pi, \hat \pi) = \OO \left( \min_{k \in [M]}  \sqrt{d_k \log(M/\delta) \over n } \cdot \E_X \| \phi_k(X, \pi(X) \|_{V_k^{-1}}  \right).
\end{align*}
However, when positive approximation error is involved, it can be cumulative in the regret bound.
Interestingly, the algorithm and proof reveal that different model classes can be selected at different states rather than choosing a single model class a priori.  Without approximation error, the optimal action and model class pair is chosen at each state and thus a slightly stronger bound can be obtained in this case. We note that Algorithm~\ref{alg::est-cov} is similar in concept to the representation selection algorithm of \cite{papini2021leveraging} -- their work is focused in the online setting, resulting in an optimistic algorithm.

\subsection{Balancing complexity and approximation error}\label{sec::approx-est}

We now consider the setting where the worst-case coverage properties are tolerable, but we would like to optimally trade-off approximation error and statistical complexity. We will examine two methods: {\bf i)} an adaptive method inspired by the SLOPE estimator~\citep{su2020adaptive} and {\bf ii)} the classical hold-out method.
While the hold-out method is desirable for its simplicity, the SLOPE method is potentially capable of achieving stronger theoretical guarantees under a slightly stronger nestedness condition on the model classes. 

Hitherto, for the dataset $D$, we required only that the actions $a_i$ are chosen independent of all potential outcomes conditional on $x_i$ without regard to any behavior policy $\mu$. We will now explicitly assume that the each $(x_i, a_i, y_i)$ in the dataset $D$ is sampled jointly from $\DD$ and a fixed behavior policy $\mu$ as discussed in Section~\ref{sec::batch-prelim}. This is stronger than before, but it is a standard setting for batch learning \cite{xie2021bellman}. Henceforth, we simply use $\E_\mu \left[ \cdot \right]$ to denote the expectation over the joint distribution $\E_{X, A \sim \DD \times \mu}   \left[ \cdot \right]$.
We also consider a relaxed version of the approximation error, which can be written in terms of the statistical approximation error between $\FF_k$ and the true reward function $f$:
\begin{align*}
    \tilde \epsilon_k = \min_{\theta \in \R^{d_k} } 2\sqrt{ \CC(\mu)  \E_\mu \left( \< \phi_k(X, A), \theta \> - f(X, A) \right)^2  }.
\end{align*}
We define $\bar \theta_k = \argmin_{\theta \in \R^{d_k}} \E_\mu \left( \< \phi_k(X, A), \theta\> - f(X, A) \right)^2$ when $\E_\mu \left[\phi(X, A) \phi(X,A)^\top \right] \succ 0$. This version of the approximation error behaves similarly to $\epsilon_k$.
If $\FF_k$ satisfies realizability, then $\tilde \epsilon_k = 0$. However, it should be noted that this version has dependence on the concentrability coefficient $\CC(\mu)$ for the worst-case dataset shift, which we are willing to tolerate in this section.

\subsubsection{SLOPE Method}
The first method, shown in Algorithm~\ref{alg::approx-est}, is inspired by the SLOPE estimator~\citep{su2020adaptive}. The algorithm begins by generating estimates $\hat \theta_k$ using the dataset $D$. It also extracts standard $\argmax$ policies $\hat \pi_k$ from these estimates --- since we are forgoing the coverage property, there is no need to employ pessimism. The main idea is to then estimate the values of the $\hat \pi_k$ policies using an improved variant of the SLOPE estimator to achieve the optimal trade-off between approximation error and complexity. We describe this sub-procedure and its differences from the original in Appendix~\ref{appendix::slope}. 

To leverage this, however, we require additional structure on the function classes in order to meet the pre-conditions of the SLOPE estimator. In particular, we assume the model classes are nested in the sense of Definition~\ref{def::nested}. Nestedness is a common paradigm in model selection in risk minimization \citep{bartlett2002model}, albeit it is not always necessary. Nestedness of linear model classes has also been assumed in model selection for online contextual bandits \citep{foster2019model}.

	\begin{algorithm}[tb]
		\caption{ SLOPE Method } \label{alg::approx-est}
		\begin{algorithmic}[1]
			\STATE \textbf{Input}: Dataset $D$, linear model classes $\FF_1, \ldots, \FF_M$, confidence parameter $\delta > 0$
			\FOR{$k \in [M]$}
			\STATE Estimate covariance matrix $V_k$ and parameters $\hat \theta_k$ as in Algorithm~\ref{alg::est-cov}.
			\STATE Set $\hat f_k(x, a) \leftarrow \<\phi_k(x,a),\hat \theta_k\>$
			\STATE Set $\hat \pi_k(x) \leftarrow \argmax_{a \in \AA} \hat f(x, a)$ with ties broken arbitrarily
			\ENDFOR
			
			\FOR{$\ell \in [M]$}
				\FOR{$ k \in [M]$ }
					\STATE Set $\hat v_{k}(\hat \pi_\ell) \leftarrow \E_X \left[  \hat f_k(X, \hat \pi_\ell(X)) \right]$
					\STATE Set $\xi_k \leftarrow \zeta_{k}(\delta/M) \cdot \E_X \max_{a } \| \phi_k(X, a) \|_{V_k^{-1}}$
					\STATE Define intervals
					\small{
					\begin{align*}
					\II_{k, \ell} = \left[\hat v_{k}(\hat \pi_\ell) - 2\xi_k , \ \hat v_{k}(\hat \pi_\ell) + 2\xi_k  \right]
					\end{align*}				
					}
				\ENDFOR
				\STATE Select model class for for evaluating $\hat \pi_\ell$:
				\small{
				\begin{align*}
				\textstyle
				\hat k(\ell) = \min \left\{   k \ : \ \bigcap_{j = k}^M \II_{j, \ell} \text{ is non-empty} \right\}
				\end{align*}
				}
				Set $\hat v(\hat \pi_\ell)\leftarrow \hat v_{\hat k(\ell)} (\hat \pi_\ell)$
			\ENDFOR
			\STATE Set $\hat k = \argmax_{k \in [M]} \hat v(\hat \pi_k)$
			\STATE \textbf{Return}: $\hat \pi = \hat \pi_{\hat k}$
		\end{algorithmic}
	\end{algorithm}

We also assume the following distributional conditions on the model classes.
\begin{assumption}\label{asmp::well-conditioned}
For all $k \in [M]$, $\E_\mu\left[ \phi(X, A) \right] = 0$ and $\Sigma_k:=  \E_\mu \left[ \phi(X, A) \phi(X, A)^\top \right] \succ 0$. Furthermore, under $\DD \times \mu$, the features $\phi_k(X,A)$ are sub-Gaussian with $\|\Sigma^{-1/2}_k \phi_k(X, A) \|_{\psi_2} \leq 1$.

\end{assumption}
\begin{assumption}\label{asmp::bounded-norm}
For all $k \in [M]$, $\| \bar \theta_k \| \leq 1$.
\end{assumption}

Centering is done for ease of exposition. The eigenvalue lower bound is useful for random design linear regression analysis \citep{hsu2012random}. The sub-Gaussian condition is standard and encompasses a large class of distributions. Assumption~\ref{asmp::bounded-norm} is like the precondition of Theorem~\ref{thm::linear} but it is distribution-dependent rather than data-dependent.

Define
\begin{align*}
\zeta_k(\delta)  = \sqrt{ \frac{\lambda d_k }{n}} +  C_1 \sqrt{ { d_k \over n} } \| V^{-1/2}_k \|  \log(4d_k/\delta) + \sqrt{\frac{C_2 d_k + C_3 d_k^{1/2} \log^{1/2}(4d_k/\delta)  + C_4  \log(4d_k/\delta)}{n}}
\end{align*}
for sufficiently large constants $C_1, C_2, C_3, C_4 > 0$ defined in Appendix~\ref{appendix::slope}. Under these assumptions, we have the following theorem.
\begin{restatable}{theorem}{thmApproxEst}
\label{thm::approx-est}
	Let $\FF_1, \ldots, \FF_M$ be a nested collection of linear model classes. For $\lambda_k = 1$ for all $k \in [M]$, Algorithm~\ref{alg::approx-est} outputs a policy $\hat \pi$ such that, with probability at least $1 - 4\delta$, for any comparator policy $\pi$,
	\begin{align*}\label{eq::approx-est}
	\regret(\pi, \hat \pi) \leq  12 \min_{k \in [M]} \left\{   \tilde \epsilon_k +   \zeta_k(\delta/M) \cdot \E_X \max_a \| \phi_k(X, a) \|_{V_k^{-1}} \right\}. 
	\end{align*}
\end{restatable}
We show the detailed proof in Appendix~\ref{appendix::slope}. Theorem~\ref{thm::approx-est} shows that it is possible to obtain a near-oracle inequality when we are willing to forgo the coverage property and focus solely on trading off approximation error and complexity. Concisely, it shows that, with high probability,
\begin{align*}
    \regret(\pi, \hat \pi) = \OO \left( \min_{k \in [M]} \tilde \epsilon_k + \sqrt{d_k \| V^{-1}_k \| \log(d_k M /\delta) \over n }  \cdot \E_X \max_{a} \|\phi_k(X, a) \|_{V_k^{-1}}  \right).
\end{align*}
There are several important properties concerning this bound. First, the approximation error is represented with $\tilde \epsilon_k$, which is potentially looser than $\epsilon_k$, but it is a natural notion of approximation error in the regression setting. Second, the coefficient $\zeta_k(\delta/M)$ on the second term is slightly larger than $\beta_{\lambda_k,\delta/M}(n, d_k)$ due to the dependence on $\| V^{-1/2}_k \|$, but they are of approximately the same order in $d$ and $n^{-1}$. 
We emphasize the factor $\E_X \max_a \| \phi_k(X, a) \|_{V_k^{-1}}$, is different from the coverage error as defined in~\eqref{eq::oracle-inequality}, which is $\E_X\| \phi_k(X, \pi(X)) \|_{V^{-1}_k}$. This version demonstrates the distribution shift effect on features, but depends on the worst-case policy rather than the comparator $\pi$. Thus, we are unable to maintain competitiveness against well-covered policies.
Finally, in Theorem~\ref{thm::approx-est}, only the approximation error $ \tilde \epsilon_{k} $ depends on $\CC(\mu)$. 

Note that $\tilde \epsilon_k$ is a weaker form of the approximation than the previously used $\epsilon_{k}(\pi, \hat \pi)$. A natural question is whether a bound of the form $\min_k\{ \tilde \epsilon_k  + \zeta_k(\delta/M) \cdot \E \| \phi(X, \pi(X)) \|_{V_k^{-1}}\}$, which satisfies all three criteria, is possible with this slightly weaker approximation error. We show in Appendix~\ref{appendix::lower-bound} that the argument in the proof of Theorem~\ref{thm::lower-bound} still applies in this case and thus a bound of this form is still not possible.

\subsubsection{Hold-out Method}

We now analyze the performance of the hold-out method, a classical model selection tool in supervised learning and risk minimization. 
It entails setting aside a fraction of the data to obtain  a sample estimate of the loss and selecting among candidates trained on the majority of the data. In batch learning, while we cannot estimate policy value without more sophisticated tools, we can estimate the regression error as a proxy. The dataset $D$ is partitioned into $D_{in}$ and $D_{out}$. We then estimate $\hat \theta_k$ for each model class individually as in Algorithm~\ref{alg::linear} and extract the policy $\hat \pi_k(x) \in \argmax_{a \in \AA} \< \phi_k(x,a), \hat \theta_k\>$, yielding
\begin{align*}
    \regret(\pi, \hat \pi_k ) 
    \leq 
    2\sqrt{  \CC(\mu) \E_\mu \left(\<\phi_k (X, A), \hat \theta_k\> - f(X, A) \right)^2 } \\       
\end{align*}
We denote the empirical regression loss on the independent hold-out set as:
\begin{align}
    \hat  L_k (\theta) & = {1 \over |D_{out}|} \sum_{(x_i, a_i, y_i) \in D_{out}} \left(\<\phi_{k}(x_i, a_i), \theta\> - y_i\right)^2
\end{align}
The hold-out method simply chooses the model class with smallest empirical loss:
\begin{align*}
    \hat k \in \argmin_{k \in [M]} \hat L_k(\hat \theta_k) 
\end{align*}

\begin{restatable}{theorem}{thmHoldOut}
\label{thm::hold-out}
    Given arbitrary linear model classes $\FF_1, \ldots, \FF_M$, let $\hat\pi = \hat \pi_{\hat k}$ where $\hat k \in \argmin_{k \in [M]} \hat L_{k}(\hat \theta_k)$. Then, there is a constant $C > 0$ such that, with probability at least $1 - 2\delta$, $\regret(\pi, \hat \pi)$ is bounded above by
    \begin{align*}
        &  \min_{k} \left\{ \tilde \epsilon_k + C\sqrt{\CC(\mu) }  \| \hat \theta_k - \bar \theta_k \|_{\Sigma_k} \right\}  \\
        & \quad + \OO \left( \sqrt{\CC(\mu)}\cdot  \frac{(1\vee  \max_\ell\| \hat \theta_\ell\|) \log^{1/2}(M/\delta) }{n_{out}^{1/4}} \right).
    \end{align*}
\end{restatable}
The detailed proof is provided in Appendix~\ref{appendix::hold-out}. 
Note that, for simplicity, we have stated the bound abstractly in terms of its estimation error $\| \hat \theta_k - \bar \theta_k \|_{\Sigma_k}$ and the norm of $\max_{\ell} \| \hat \theta_\ell\|$, where $\Sigma_{k}$ is defined in Assumption~\ref{asmp::well-conditioned}. Standard analyses of random design linear regression \citep{hsu2012random} yield high probability upper bounds on the first term on the order of $O(\sqrt{d_k/n_{in}})$ for sufficiently large $n_{in}$, as expected. Similarly, $\| \hat \theta_\ell\| \approx \| \bar \theta_\ell \|$ up to constant additive error for $n_{in}$ large enough.

The main takeaway from this model selection guarantee is that, while we are able to select to achieve error on the order of the best model class, this is only achieved when the estimation error depends on concentrability coefficient $\CC(\mu)$. Additionally, there is some residual estimation error on the order of $O(1/n_{out}^{1/4})$ due to the hold-out method itself, which is slower than the typical $O(1/\sqrt n)$; however,  we note that this term does not have any dependence on $d$, assuming $\|\bar \theta_\ell\|$ is of constant size.

\paragraph{Balancing approximation error and coverage.} We conclude by remarking that a final case may be considered when we ignore the model selection criterion of statistical complexity and aim to balance only approximation error and coverage. In this case, the problem becomes trivial since we are ``permitted'' to take arbitrarily large model classes until realizability is achieved.

\section{Experiments}\label{sec::exp}

In this section, we present preliminary synthetic experiments in order to study the utility of the above model selection algorithms empirically and compare them. We analyze individually the complexity-coverage trade-off and the approximation error-complexity trade-off.

In both experiments, we generated a collection of feature maps, each defining a linear model class. We first evaluated the performance of the base algorithms with each model class. We then compare this performance to the proposed selection algorithms, which leverage all the model classes. All results were averaged over 20 trials. Error bands, representing standard error, are shown only for the model selection algorithms for clarity. Further implementation details can be found in the appendix.

\subsection{Complexity-Coverage Trade-off}

\begin{figure}
\centering
\includegraphics[width=7cm]{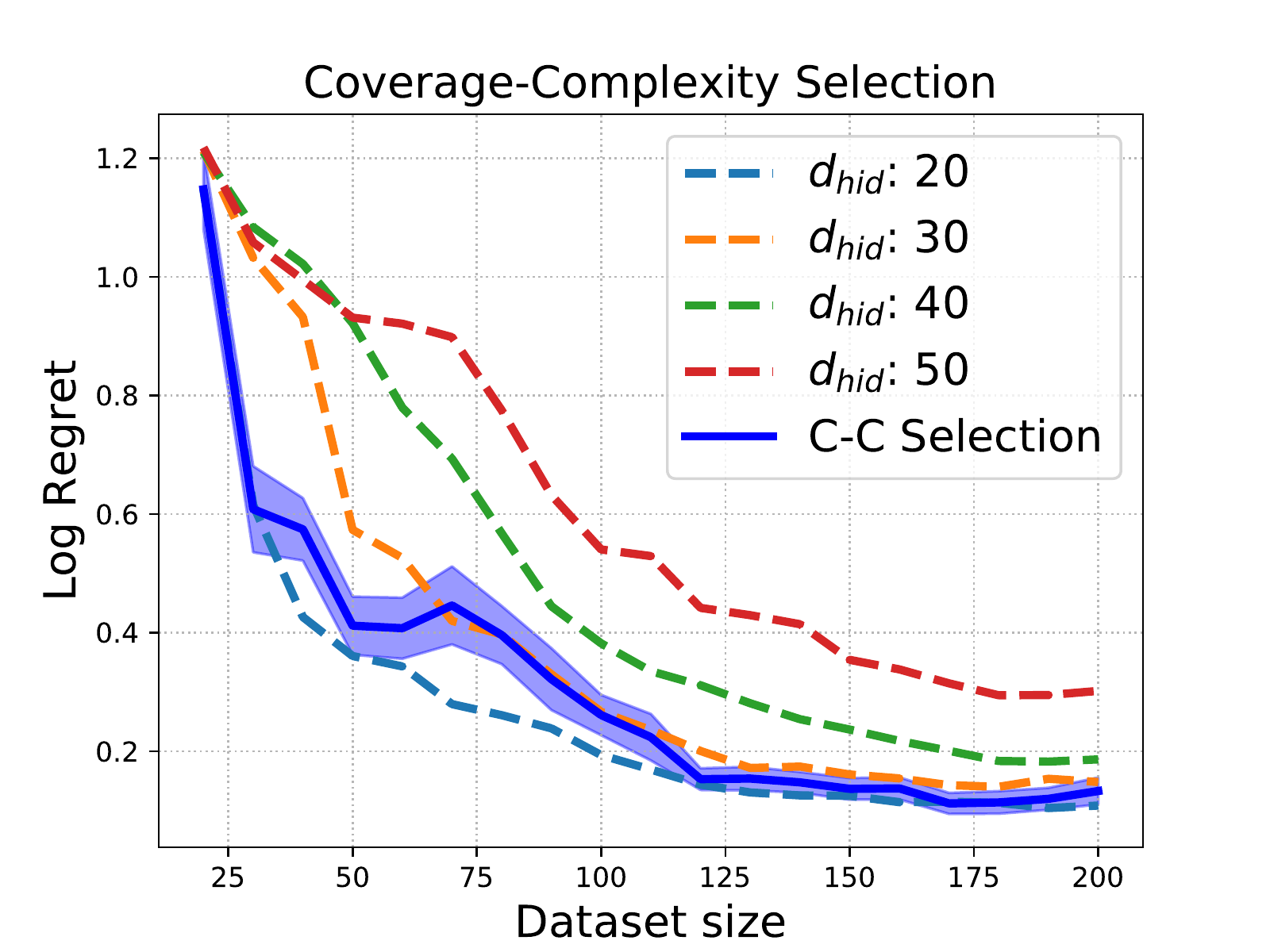}
\caption{In the fully realizable case, the performance of the Algorithm~\ref{alg::est-cov} for model selection is compared to base algorithms that only use a single model class. Each model class is defined by a different feature representation generated with underlying dimension $d_{hid}$. The error band represents standard error. }\label{fig::cc}
\end{figure}

In this setting, we studied the case where all feature maps are capable of realizing the true reward function (i.e., no approximation error). That is, the learner need not deal with any approximation error, but it can benefit from selecting a good representation. We let $|\XX| = 20$ and $|\AA| = 10$ and generated $d = |\XX||\AA|$-dimensional feature maps through the following procedure: for model class $k$ a random collection of $d_{hid, k}$ vectors of size $|\XX||\AA|$ are generated ensuring that a linear combination exactly equals $f$. We then randomly project them to $d$-dimensions to produce $\phi_k$. As the base algorithm for each $\phi_k$, we used Algorithm~\ref{alg::linear} and for model selection, we  implemented Algorithm~\ref{alg::est-cov}.

Figure~\ref{fig::cc} compares the performance. As expected, the model class with smallest $d_{hid}$ performs best. The model selection is nearly able to match this performance, even without knowing which of the feature maps corresponds to the smallest $d_{hid}$.

\subsection{Approximation Error-Complexity Trade-off}

\begin{figure}
\centering
\includegraphics[width=7cm]{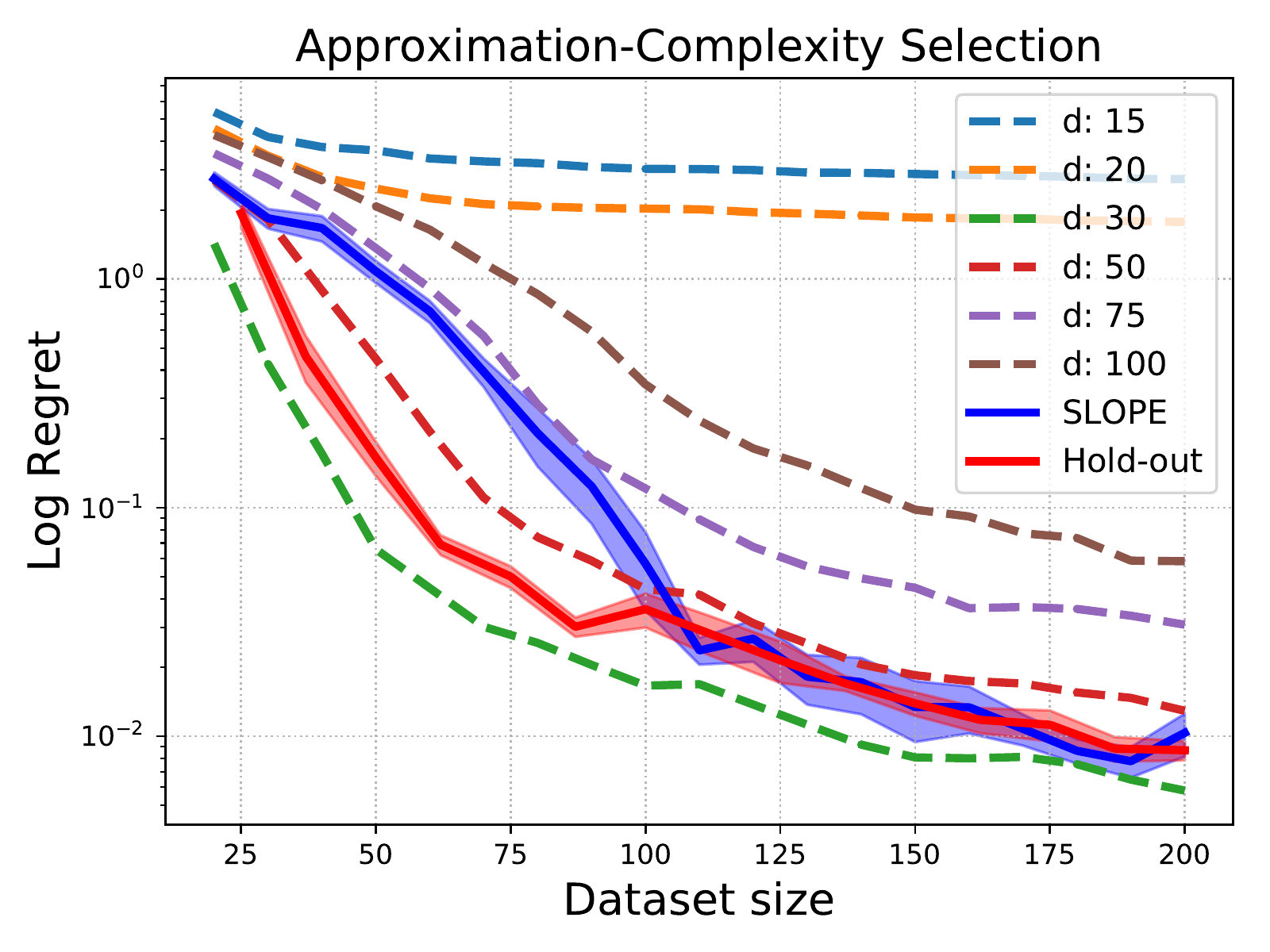}
\caption{The performance of the SLOPE and hold-out methods are compared against base algorithms that only use a single model class of varying dimension. Both are eventually able to nearly match the performance of the best model class, but the hold-out method is consistently better with less data. Error bands represent standard error.}\label{fig::ac}
\end{figure}

Next, we consider trading off approximation error and complexity with nested function classes. We again let $|\AA| = 10$, but allowed $\XX$ to be infinite with feature vectors generated from zero-mean normal distributions with different covariance matrices. The feature vectors were given ambient dimension $d = 100$, but the reward function was designed using only the first $d_* = 30$ coordinates. Model classes were generated by truncating full feature vectors to the following dimensions $\{15, 20, 30,  50, 75, 100 \}$. Thus, the first two suffer from approximation error while the last three are excessively large.

As base algorithms, we implemented Algorithm~\ref{alg::linear} again. For model selection, we considered both the SLOPE method and the hold-out method with an 80/20 data split. Figure~\ref{fig::ac} compares the performance of all algorithms. We find that both model selection algorithms are eventually able to match performance of the best model class. Interestingly, the hold-out performs consistently better than the SLOPE method with small data.

\section{Discussion}

In this paper, we introduced the theoretical study of model selection for batch policy optimization, identifying three sources of error to consider when selecting model classes. We showed that balancing all three is not possible in general while remaining competitive with an oracle, but relaxing any one allows the design of effective model selection algorithms. 
Several open questions remain. First, the results thus far have applied only to the contextual bandit setting. While this is useful to gain initial intuition for the problem, we expect that the challenges become considerably more complex for reinforcement learning and it remains to understand what is possible there. Another interesting direction is to understand more formally the performance of the hold-out method. The theoretical results of Section~\ref{sec::approx-est} defer to a worst-case concentrability coefficient on the estimation error to handle the dataset shift. This is seemingly worse than the covariance penalty paid by the other studied algorithms, but the empirical results seem to suggest that the hold-out method is effective and more robust in practice. It would be interesting to further understand these observations. Finally, the results thus far have addressed only linear model classes, but we hypothesize that similar trade-offs are likely to be observed for general function classes that handle coverage in terms of comparator-specific concentrability coefficients \cite{xie2021bellman,uehara2021pessimistic}.

\section*{Acknowledgements}
We thank Christoph Dann and anonymous reviewers for their detailed and helpful feedback in improving the paper. JNL is partially supported by NSF GRFP.

\bibliographystyle{alpha}
\bibliography{main}
\raggedbottom

\appendix

\onecolumn

\section{Sub-Gaussian and Sub-Exponential Random Variables}\label{sec::subg}

In this section, we review basic definitions and properties of sub-Gaussian and sub-exponential random variables. See \cite{vershynin2018high} for a comprehensive introduction.

Let $X \in \R$ be a random variable. We define the norms:
\begin{align}
\| X \|_{\psi_2} & := \sup_{p \in \N} p^{-1/2} \left(  \E | X | ^p \right)^{1/p} \\
\| X\|_{\psi_1} & := \sup_{p \in \N} p^{-1} \left(  \E | X|^p \right)^{1/p}
\end{align}

\begin{definition}
	The random variable $X$ is sub-Gaussian with parameter $\|X\|_{\psi_2}$ if $\|X\|_{\psi_2} < \infty$. It is sub-exponential with parameter $\| X\|_{\psi_1}$ if $\| X \|_{\psi_1} < \infty$. 
\end{definition}
For a non-negative real value $\tau \geq 0$ we write $X \sim \subg(\tau^2)$ to indicate that $ \|X\|_{\psi_2} \leq \tau$. Similarly, we write $X \sim \sube(\tau)$ to suggest $\|X\|_{\psi_1} \leq \tau$.
We note that this definition of sub-Gaussian random variables is equivalent up to constant factors with an alternative popular definition when $\E X = 0$. This definition requires that $\E e^{\lambda X} \leq \exp \left(  { \lambda^2 \tau^2  \over 2}  \right)$ for all $\lambda \in \R$.
Let $X \in \R^d$ be a random vector. Then, we write $\| X \|_{\psi_2} = \sup_{v  \in \R^d \ : \ \| v\| \leq 1}  \| v^\top X\|_{\psi_2}$. The same notational conventions above apply to the vector $X$.

We now state several basic results concerning sub-Gaussian and sub-exponential random variables that will be used throughout the remaining proofs.

\begin{lemma} [\cite{vershynin2010introduction}, Lemma 2.7.7]\label{lem::sube-product} Let $X$ and $Y$ be (potentially dependent) real-valued random variables. Then, the following holds: $\| X Y \|_{\psi_1} \leq \| X \|_{\psi_2} \| Y \|_{\psi_2}$.
\end{lemma}

\begin{lemma}\label{lem::moment-bound}
	Let $X \in \R^d$ be a random vector with second moment matrix $\Sigma = \E XX^\top$. Let $v \in \R^d$ be such that $\|v\| \leq 1$. Then, $v^\top \Sigma v \leq 2 \|X\|_{\psi_2}^2$.
\end{lemma}
\begin{proof}
	We have $v^\top \Sigma v = \E v^\top XX^\top v = \E (v^\top X)^2$. From the definition of $\|\cdot \|_{\psi_2}$, we have that $ \E | v^\top X|^2  \leq  2 \| v^\top X \|_{\psi_2}^2$. Finally, we note that $\| v^\top X \|_{\psi_2} \leq \| X \|_{\psi_2}$ since $\|v\| \leq 1$.
\end{proof}

\begin{lemma}\label{lem::subg-alt-def}
    Let $X$ satisfy $\E X = 0$ and $\|X \|_{\psi_2} \leq \tau$. Then, $\E \exp \left(\lambda X \right)\leq \exp \left( {5\lambda^2 \tau^2 \over 2}\right) $ for all $\lambda \in \R$.
\end{lemma}
\begin{proof}
    Note that $\|X\|_{\psi_2} \leq \tau$ implies that $\left(\E | X|^{p}\right)^{1/p} \leq \tau \sqrt{p}$. for all $p$. Theorem 3.10 of \cite{duchi2019lecture} shows that this implies the stated condition.
\end{proof}

\section{Proof of Theorem~\ref{thm::linear}}\label{sec::linear-proof}

\thmLinear*

We first leverage a basic result in the analysis of fixed design linear regression problems. For convenience, we will write $\phi_i = \phi(x_i, a_i)$, $f_i = f(x_i, a_i)$ and noise $\eta_i = \eta_i(a_i)$. That is, $y_i = f_i + \eta_i$. Recall the following definitions:
\begin{align}
V & =  {\lambda \over n } \I_d + {1 \over n} \sum_{ i \in [n]} \phi_i \phi_i^\top \\
\hat \theta & =  V^{-1} \left(  {1 \over n} \sum_{ i \in [n]} \phi_i y_i  \right)\\
\theta_* & \in \argmin_{\theta \in \R^d} {1 \over n} \sum_{i \in [n]} \left( \phi_i^\top \theta - f_i   \right)^2
\end{align}

\subsection{Concentration}
\begin{lemma}\label{lem::linear-concentration}
	Conditioned on $(x_i, a_i)_{i \in [n]}$, with probability at least $1 - \delta$,
	\begin{align}
	\| \hat \theta - \theta_*\|_V  \leq \sqrt{ \frac{ \lambda \| \theta_*\|^2 } {n}}  + \sqrt{ \frac{ C_1 d + C_2 {d^{1/2} \log^{1/2}(1/\delta)} + C_3 \log(1/\delta) }{n} } 
	\end{align}
	where $C_1 = 5$, $C_2 = 10$, and $C_3 = 10$.
\end{lemma}
\begin{proof}
	From the definition of $\hat \theta$, we have 
	\begin{align}
	\| \hat \theta - \theta_*\|_{V} & = \| {1 \over n } V^{-1} \sum_i \phi_i \left( f_i + \eta_i  \right) - \theta_* \|_V \\
	& = \| {1 \over n } V^{-1} \sum_{i \in [n]} \phi_i \phi_i^\top \theta_* + {1 \over n }V^{-1} \sum_{i} \phi_i \eta_i - \theta_* \|_V 
	\end{align}
	where in the last equality we have used the fact that $\theta_*$ is the solution to $\min_{\theta} \sum_{i}  ( \phi_i^\top \theta_* - f_i)^2$ and therefore satisfies the normal equations:
	\begin{align}
	\sum_{i} \phi_i \phi_i^\top \theta_* = \sum_{i} \phi_i f_i
	\end{align}
	Then,
	\begin{align}
	\| \hat \theta - \theta_*\|_V & = \| - { \lambda \over n} V^{-1} \theta_* +  { 1\over n }V^{-1} \sum_i \phi_i \eta_i \|_V \\ 
	& \leq \| { \lambda \over n} V^{-1} \theta_* \|_V   +  \| { 1\over n }V^{-1} \sum_i \phi_i \eta_i \|_V \\
	& \leq  \sqrt{ { \lambda \over n} } \| \theta_*\| + \|{ 1\over n } V^{-1/2} \sum_{i} \phi_i \eta_i \|
	\end{align}
	where the last inequality follows since $\sigma_{\max}^{1/2} (V^{-1}) \leq \left(\lambda / n\right)^{-1/2}$.

	To bound the second term, we apply the Lemma~\ref{lem::hanson-wright-application}, stated below, which is an application of standard fixed-design linear regression results of \cite{hsu2012random}. This shows that 
	\begin{align*}
	    \|{ 1\over n } V^{-1/2} \sum_{i} \phi_i \eta_i \|^2 \leq \frac{ 5 d + 10 \sqrt{d \log(1/\delta) } + 10  \log(1/\delta) }{n }
	\end{align*}
	with probability at least $1 - \delta$. Applying this to the previous bound on $\| \theta - \theta_*\|_V$ gives the result.
	\end{proof}
	\begin{lemma}\label{lem::hanson-wright-application}
	\begin{align*}
	    P \left( \|{ 1\over n } V^{-1/2} \sum_{i} \phi_i \eta_i \|^2 > \frac{5 d + 10 \sqrt{d \log(1/\delta) } + 10  \log(1/\delta)}{n}  \ | \ x_{1:n}, a_{1:n} \right) \leq \delta  
	    \end{align*}
	\end{lemma}

    \begin{proof}[Proof of Lemma~\ref{lem::hanson-wright-application}]

    Let $\eta = (\eta_1, \ldots, \eta_n)^\top$ and $\Phi = \begin{bmatrix}
	\phi_1 & \cdots & \phi_n
	\end{bmatrix}^\top $. Note then that $\|{ 1\over n } V^{-1/2} \sum_{i} \phi_i \eta_i \| = {1 \over \sqrt n }\|  \left(  \lambda \I_d + \Phi^\top \Phi \right)^{-1/2}  \Phi^\top  \eta \| $, where $\eta$ is a random $1$-sub-Gaussian random vector consisting of independent entries.  Note that we have that $\E\eta_i = 0$ and $\| \eta_i \|_{\psi_2} \leq 1$, which, by Lemma~\ref{lem::subg-alt-def} implies $\E \exp \left(\lambda X \right) \leq \exp \left( { \lambda^2 \sigma^2 \over 2}\right)$ where $\sigma^2  \leq 5$. The concentration result then follows from a version of the Hanson-Wright inequality due to \cite{hsu2012tail}, a restatement of which can be found in Lemma~\ref{lem::hanson-wright}:
	\begin{align}
	P_\eta \left(  \| A \eta \|^2 > C_1 \tr( AA^\top ) + C_2 \sqrt{\tr ((AA^\top)^2 )   \log(1/\delta)} + C_3\sigma^2 \| AA^\top \|\log(1/\delta) \right) \leq \delta
	\end{align}
	where $A =\left(  \lambda  \I_d + \Phi^\top \Phi \right)^{-1/2}  \Phi^\top$ and  $C_1 = 5$, $C_2 = 10$, and $C_3 = 10$. We may then bound the relevant quantities. Let $\Phi^\top \Phi = U \Lambda U^\top$ be the spectral decomposition of $\Phi^\top \Phi$ where $U$ is unitary and $\Lambda \succeq 0$ is diagonal.
	\begin{align}
	\tr(AA^\top) & = \tr \left( \left(  \lambda  \I_d + \Phi^\top \Phi \right)^{-1/2}  \Phi^\top \Phi \left(  \lambda  \I_d + \Phi^\top \Phi \right)^{-1/2}  \right) \\
	& = \tr \left( \left(  \lambda  \I_d + \Phi^\top \Phi \right)^{-1}  \Phi^\top \Phi   \right)  \\
	& \leq \tr \left( ( \lambda  \I_d + \Lambda)^{-1} \Lambda  \right) \\
	& \leq d \\
	\tr((AA^\top)^2) & = \tr \left( \left(  \lambda  \I_d + \Phi^\top \Phi \right)^{-1/2}  \Phi^\top \Phi \left(  \lambda  \I_d + \Phi^\top \Phi \right)^{-1} \Phi^\top \Phi \left(  \lambda  \I_d + \Phi^\top \Phi \right)^{-1/2} \right) \\ 
	& \leq \tr \left(  \left(  \lambda\I_d  + \Lambda \right)^{-1} \Lambda ( \lambda\I_d + \Lambda)^{-1}   \Lambda  \right) \\
	& \leq d
	\\
	\| A A^\top \| & = \| \left(  \lambda  \I_d + \Phi^\top \Phi \right)^{-1/2}  \Phi^\top \Phi \left(  \lambda  \I_d + \Phi^\top \Phi \right)^{-1/2} \| \\
	& = \| U ( \lambda \I_d + \Lambda)^{-1/2} \Lambda ( \lambda \I_d + \Lambda )^{-1/2} U^\top\| \\
	& \leq 1 
	\end{align}
	where the very last inequality uses the that the unitary matrix preserves the norm an the maximal eigenvalue of $( \lambda \I_d + \Lambda)^{-1/2} \Lambda ( \lambda \I_d + \Lambda )^{-1/2}$ is at most $1$.
	 
	We therefore conclude that
	\begin{align}
	\| { 1 \over n} V^{-1/2} \sum_i \phi_i \eta_i \|^2  & \leq \frac{ C_1 d + C_2 \sqrt{d \log(1/\delta)} + C_3 \log(1/\delta) }{n}
	\end{align}
	with probability at least $1 - \delta$. 
        
    \end{proof}

\subsection{Full Proof}
\begin{proof}[Proof of Theorem~\ref{thm::linear}] for this proof, all expectations $\E$  denote $\E_X$, the expectation over the state random variable $X$ from $\DD$.
	By adding and subtracting, the regret may be decomposed simply as 
	\begin{align}
	\regret(\pi, \hat \pi) & = \E \left[  f(X,\pi(X)) - \< \phi(X, \pi(X)), \theta_* \>\right] 
	\\ &\quad 
	+ \E\left[  \<\phi(X, \hat \pi(X)), \theta_*\> - f(X, \hat \pi(X))\right] 
	\\ & \quad 
    + \E \left[  \< \phi(X, \pi(X)) ,\theta_*\> - \< \phi(X, \hat \pi(X)), \theta_*\> \right] \\
	& \leq  \epsilon(\pi, \hat \pi)  + \E \left[  \< \phi(X, \pi(X)) ,\theta_*\> - \< \phi(X, \hat \pi(X)), \theta_*\> \right]
	\end{align}
	For the remainder of the proof, we focus on bounding the second term.  Adding and subtracting again, we have
	\begin{align}
	 &\E \left[  \< \phi(X, \pi(X)) - \phi(X, \hat \pi(X)) ,\theta_*\> \right] \\
	 & \leq \E \left[   \phi(X, \pi(X))^\top \theta_* -  \phi(X, \hat \pi(X))^\top \hat \theta \right] 
	 + \E \left[ \phi(X, \hat \pi(X))^\top \hat \theta -    \phi(X, \hat \pi(X))^\top \theta_* \right] \\
	 & \leq  \E \left[   \phi(X, \pi(X))^\top \theta_* -  \phi(X, \hat \pi(X))^\top \hat \theta \right] 
	 +  \| \hat \theta - \theta_*\|_{V} \cdot \E \| \phi(X, \hat \pi(X)) \|_{V^{-1}}
	\end{align}
where the last inequality is due to Cauchy-Schwarz.
	
	Now, we apply the result of Lemma~\ref{lem::linear-concentration} to get that the event
	\begin{align}
	\| \hat \theta - \theta_*\|_{V} & \leq \sqrt{ \frac{ \lambda \| \theta_*\|^2 } {n}}  
	+ \sqrt{ \frac{ C_1 d + C_2 {d^{1/2} \log^{1/2}(1/\delta)} + C_3 \log(1/\delta) }{n} }  \\
	& \leq \beta_{\lambda, \delta} (n, d) 
	\end{align}
	occurs with probability at least $1 - \delta$ for absolute constants $C_1, C_2, C_3 > 0$ defined there. Conditioning on this event, we have
	\begin{align*}
	\E \left[  \< \phi(X, \pi(X)) - \phi(X, \hat \pi(X)) ,\theta_*\> \right] 
	& \leq	\E \left[   \phi(X, \pi(X))^\top \theta_* -  \phi(X, \hat \pi(X))^\top \hat \theta \right] 
	\\
	& \quad 
	+ \beta_{ \lambda, \delta}(n,d) \cdot \E \| \phi(X, \hat \pi(X)) \|_{V^{-1}} \\
	& \leq \E \left[   \phi(X, \pi(X))^\top \theta_* -  \phi(X,  \pi(X))^\top \hat \theta \right] \\
	& \quad + \beta_{\lambda, \delta} (n, d) \cdot \E \| \phi(X,  \pi(X)) \|_{V^{-1}} \\
	& \leq \left( \| \hat \theta - \theta_* \|_{V} + \beta_{\lambda, \delta} (n, d) \right) \cdot \E \|\phi(X, \pi(X)) \|_{V^{-1}} \\
	& \leq 2 \beta_{\lambda, \delta} (n, d) \cdot \E \|\phi(X, \pi(X)) \|_{V^{-1}}
	\end{align*} 
	where the second inequality applies the penalized action-selection for policy $\hat \pi$, the third inequality applies Cauchy-Schwarz, and the last inequality once again applies the condition on the concentration of $\| \hat \theta - \theta_* \|_{V}$.
\end{proof}

\subsection{Discussion of Approximation Error}
In this paper, we work with a fairly general notion of approximation error $\epsilon_k(\pi, \hat \pi)$. Note that this depends both on the comparator policy $\pi$ and the learned policy $\hat \pi$ and it tends to be small when $\theta_*$ outputs similar rewards to $f$ on both of these policies. The reason for this choice is that it allows a large degree of flexibility as many natural alternatives may upper bound it, for example those given below.

Here we point out a couple alternatives that appear frequently in bandit and RL theory.
\begin{enumerate}
    \item Perhaps the most common assumption is a worst-case difference between $f$ and the model class $\FF_k$ \citep{jin2019provably,foster2020beyond}:
    \begin{align*}
        \epsilon_{k, \text{worst-case}} = \min_{\hat f \in \FF_k} \sup_{x \in \XX, a \in \AA} | f(x, a) - \hat f(x, a) |
    \end{align*}
    The obvious disadvantage of this version is that certain states or contexts might be irrelevant but still lead to large prediction errors. Furthermore, the  minimizing $\hat f$ does not generally have any convenient statistical properties (e.g. satisfying first-order optimality conditions in the linear case).
    
    \item Versions of the minimum squared error are also commonly used \citep{chen2019information}:
    \begin{align*}
        \epsilon_{k, \text{sq}}  = \min_{\hat f \in \FF_k} \E_\mu \left( \hat f(X, A) - f(X,A) \right)^2
    \end{align*}
    This is a natural formulation from a statistical perspective as well and it partially remedies some of the problems of the worst-case approximation error since we care only about those states and actions induced under $\mu$. Unfortunately, this typically brings a concentrability coefficient into the mix. We leverage this as an upper bound in Section~\ref{sec::approx-est}.
\end{enumerate}
We remark that numerous prior works make the assumption that an upper bound on the approximation error is \textit{known}. However, it is unrealistic in practice to expect that such information is available, and it furthermore trivializes the model selection problem. While we consider upper bounds to $\epsilon_k(\pi, \hat \pi)$ in Section~\ref{sec::approx-est}, we make no such assumption about knowing the value of this upper bound.

\section{Proof of Theorem~\ref{thm::lower-bound}}\label{appendix::lower-bound}

\thmLowerBound*

Let us begin by first describing the dataset and observations for a given contextual bandit instance. The set of contextual bandit instances is determined by possible distributions $\DD$ over state-reward pairs. For our construction, we fix the state-action pairs in the dataset $\{(x_i, a_i)\}_{i = 1}^n$ across all instances that we consider. To be clear, the distribution of the rewards in the dataset under different instances will be different, but the covariates are held fixed. Note that this is not necessary for the lower bound, but it will suffice in our example.

\begin{proof}[Proof of Theorem~\ref{thm::lower-bound}]

The main idea of the proof is to construct a difficult contextual bandit problem with $\AA = \{a_1, a_2\}$ and show that the oracle can leverage a pair of model classes satisfying Definition~\ref{def::nested} to achieve small regret. The hardness of the contextual bandit problem will be shown via a reduction to a multi-armed bandit (MAB). The model classes will be chosen such that one is well-specified while the other has no approximation error in some instances but large approximation error in others.

To start, consider a class of two two-armed bandit instances $\EE = (\nu_1, \nu_2)$ (i.e. no states) which are identified by their product distributions over rewards of both arms. We let $\nu_1 = \NN(-\Delta, 1) \times \NN(-2\Delta, 1)$ and $\nu_2 = \NN(-\Delta, 1) \times \NN(0, 1)$ for $\Delta > 0$ to be determined later. That is, across both instances, arm $a_1$ has the same reward distribution, but arm $a_2$ can have either mean $0$ or $-2\Delta$. We let $\E_{\nu_{i}}$ denote the expectation associated with instance $\nu_i$.  We let $D$ be a dataset consisting of $n_1 > 0 $ samples from $a_1$ and $n_2 > 0$ samples from $a_2$ with $n_1$ and $n_2$ to be determined precisely later. Such a construction is similar to that of standard lower bounds in bandits \citep{bubeck2013bounded,lattimore2018bandit}; however, since we are in the offline setting, the dataset is given.

We now establish a regret lower bound for any arbitrary algorithm $A$ that outputs an arm $A(D) \in \{a_1, a_2\}$ as a function of the data $D$. The next lemma follows from a standard application of Le Cam's two-point method and similar results for the offline multi-armed bandit problem \citep{xiao2021optimality}.

\begin{lemma}\label{lem::mab-lower-bound}
	Let $\Delta = { 1 \over 2\sqrt{n_2}}$. Then, for any algorithm $A$,
	$
	\max_{i, j}\E_{\nu_i} \left[ Y(a_j) - Y(A(D))  \right] \geq {1 \over 8\sqrt n_2}
	$
\end{lemma}
Henceforth, we will define $\Delta:= {1 \over 2 \sqrt{n_2}}$.
We now construct a linear contextual bandit instance and apply a reduction to the MAB setting so that we may leverage the stated lower bound. Let $\XX$ be a singleton (that is, states have no effect) and again $\AA = \{ a_1, a_2\}$. Since there is only a single state, we omit notational dependence of policies\footnote{For a deterministic policy $\pi$, we simply use $\pi \in \AA$ to denote the selected action and $n_\pi$ denotes the number of samples to arm $\pi$.} and functions on the state.

We again consider two instances $\EE = \{ \nu_1, \nu_2\}$ which each govern the data distribution denote by $\DD_{\nu_i}$. For $\DD_{\nu_1}$, we set $Y \sim \NN(-\Delta, 1) \times \NN(-2\Delta, 1)$ and, for $\DD_{\nu_2}$, we set $Y \sim \NN(-\Delta, 1) \times \NN(0, 1)$ where $\Delta$ is defined above. Note that this ensures that the noise for either instance is given by the centered standard normal distribution $(Y(a) - f(a)) \sim \NN(0, 1)$ for all $a \in \AA$. We use $\pi_*$ to denote the optimal policy (action), which depends on the instance. In $\nu_1$, we have $\pi_* = a_1$ and in $\nu_2$, we have $\pi_* = a_2$. Finally, we assume that the batch dataset $D$ again consists of $n_1> 0$ samples of $a_1$ and $n_2 > 0$ samples of $a_2$ with $n_1 \geq n_2$ giving a total of $n = n_1 + n_2$ samples. Exact quantities will be determined at the end. 

Next, we construct two linear model classes $\FF_1$ and $\FF_2$. For $\FF_1$, we use the following $1$-dimensional feature map $\phi_1 :\AA \to \R$:
\begin{align*}
    \phi_1(a) = \begin{cases} 1 & a = a_1 \\
    0 & a = a_2
    \end{cases}.
\end{align*}
$\FF_1$ thus has some opportunity to make predictions about the mean of $a_1$ but the features are trivial for $a_2$, potentially leading to approximation error. For $\FF_2$, we set $\phi_2: \AA \to \R^2$ as 
\begin{align*}
    \phi_2(a) = \begin{cases}
        (1, 0)^\top & a = a_1 \\
        (0, 1)^\top  & a = a_2 
    \end{cases}
\end{align*}
Note that this model is well-specified as $f(a) = \phi_2(a)^\top \theta$ by setting $\theta = (f(a_1), f(a_2))^\top$. It is also evident that $\FF_1$ and $\FF_2$ are nested according to Definition~\ref{def::nested}.

Let \begin{align*}
    \theta_{k, *} = \argmin_{\theta \in \R^{d_k}} \sum_{i \in [n]} \left( \phi_k(a^{(i)})^\top \theta - f(a^{(i)} ) \right)^2
\end{align*} for $k \in \{1, 2\}$ where $a^{(i)}$ denotes the action of  the $i$th datapoint in $D$. It is easy to verify that the following conditions are true as long as $n_1 > 0$ and $n_2 > 0$:
\begin{enumerate}
    \item In $\nu_1$: $\theta_{1, *} = -\Delta$ and $\theta_{2, *} = (-\Delta, -2\Delta)$.
    \item In $\nu_2$: $\theta_{1, *} = -\Delta$ and $\theta_{2, *} = (-\Delta, 0)$.
\end{enumerate}

We now summarize the estimation error and approximation error for both model classes. The contribution of estimation error follows directly from the definitions. We summarize the results in the following fact. Note that we need not include expectations over the state since there is only one state.
\begin{fact}
    Let $V_k = {\lambda \over n } + {1 \over n} \sum_{i \in [n]} \phi_k(a^{(i)}) \phi_k(a^{(i)})^\top$ for $k \in \{1, 2\}$ and $\lambda > 0$. For any instance in $\EE$ and comparator $\pi \in \AA$, the following inequalities hold:
    \begin{align*}
        \| \phi_1(\pi) \|_{V_1^{-1}}\leq \sqrt{ n \over n_1 }  & &  \| \phi_2(\pi) \|_{V_2^{-1}} \leq \sqrt{ n \over n_\pi }
    \end{align*}
\end{fact}
\begin{proof}
    The results follow by direct calculation and using the fact that $n_1, n_2 > 0$.
\end{proof}

For the approximation error, we can clearly see that, for model class $\FF_2$, $\epsilon_2(\pi, \hat \pi) = 0$ for all instances in $\EE$ and all $\hat \pi$ and $\pi$ since the model is well-specified. For $\FF_1$, the approximation error will depend on the instance. Observe that in $\nu_2$, we have
\begin{align*}
    \epsilon_1(\pi, \hat \pi) & = | \<\phi_1 (\hat \pi), \theta_{1, *} \>  - f(\hat \pi) | + | \<\phi_1(\pi), \theta_{1, *} \> - f(\pi) | = 0
\end{align*}
regardless of what $\pi$ and $\hat \pi$ are. Furthermore, for $\nu_1$, we have $\epsilon_1(\pi, \hat \pi) \leq 2\Delta$ in the worst case.

Combining the results for both estimation error and approximation error, we have
\begin{align*}
\epsilon_1(\pi_*, \hat \pi) + \sqrt{\frac{d_1}{n}} \cdot \E_X \| \phi_1(\pi_*) \|_{V_1^{-1}} & \leq 2\Delta + \frac{1}{\sqrt {n_1} } \leq {2 \over \sqrt{n_2}} \\
\epsilon_2(\pi_*, \hat \pi) + \sqrt{\frac{d_2}{n}} \cdot \E_X \| \phi_2(\pi_*) \|_{V_2^{-1}} & \leq \sqrt{2 \over n_1}
\end{align*}
in instance $\nu_1$ and
\begin{align*}
    \epsilon_1(\pi_*, \hat \pi) + \sqrt{\frac{d_1}{n}} \cdot \E_X \| \phi_1(\pi_*) \|_{V_1^{-1}} & \leq \frac{1}{\sqrt {n_1} } \\
\epsilon_2(\pi_*, \hat \pi) + \sqrt{\frac{d_2}{n}} \cdot \E_X \| \phi_2(\pi_*) \|_{V_2^{-1}} & \leq \sqrt{2 \over n_2}
\end{align*}
in instance $\nu_2$. Therefore, we conclude that 
\begin{align*}
    \min_{k \in [2]} \left\{ \epsilon_k(\pi_*, \hat \pi) + \sqrt{d_k \over n} \cdot \E_X \| \phi_k(X, \pi_*) \|_{V_k^{-1}}\right\} \leq \frac{2}{\sqrt{n_1}} 
\end{align*}
for all $\nu \in \EE$. Furthermore, note that this contextual bandit setting exactly reduces to the MAB problem and thus Lemma~\ref{lem::mab-lower-bound} requires that the regret of any algorithm $A$ be lower bounded as $\E_{\nu} \left[  \regret(\pi, A(D))  \right] \geq \frac{1}{8 \sqrt{n_2 }} $ for some $\nu \in \EE$ with our given choice of $\Delta$. Therefore, there is a constant $C_1 > 0$ such that for any algorithm $A$, there exists $\nu \in \EE$ satisfying
\begin{align*}
\frac{ \E_{\nu} \left[  \regret(\pi_*, A(D))  \right] }{ \min_{k \in [2]}  \left\{ \epsilon_{k}(\pi_*, \hat \pi_k) + \sqrt{\frac{d_k}{ n }} \cdot\E_X \|\phi_k(X, \pi_*) \|_{V_k^{-1}}\right\}  }& \geq \sqrt{ \frac{ n_1}{ 64n_2} }
\end{align*}
Finally, we are left with choosing $n_1$ and $n_2$ as the number of samples in the dataset (which is the same across all of the instances). Choosing $n_1 =  \Omega  \left( \alpha^2 n_2 \right)$  ensures the claim, which is possible since it was assumed that $n = \Omega(\alpha^2)$.

\end{proof}

 \subsection{Proof of Lemma~\ref{lem::mab-lower-bound}}
 \begin{proof}

Note that
	\begin{align}
		\max_{i, j}\E_{\nu_i} \left[ Y(j) - Y(A(D))  \right] = \max_{i}\E_{\nu_i} \left[ Y(i)  - Y({A(D)})  \right]
		\end{align} by definition of the instances $\nu_1$ and $\nu_2$. For convenience, we just write $A$ instead of $A(D)$. Then,
		\begin{align}
		\max_{i}\E_{\nu_i} \left[ Y(i) - Y(A)  \right] & \geq \frac{1}{2} \left(  \E_{\nu_1} \left[ Y(1) - Y(A)  \right] + \E_{\nu_2} \left[ Y(2) - Y(A)  \right] \right) \\
		& = \frac{\Delta}{2} \left(  P_{\nu_1}(A \neq 1) + P_{\nu_2}(A \neq 2) \right)	 \\
		& \geq \frac{\Delta}{2} \left(1 - \| P_{\nu_1} - P_{\nu_2} \|_{TV} \right) \\
		& \geq \frac{\Delta}{2} \left(1 - \sqrt{ {1 \over 2} D_{KL} ( P_{\nu_1} \| P_{\nu_2 }) } \right)
		\end{align} 
where the last two inequalities follow from the definition of the total-variation distance and Pinsker's inequality, respectively. Then, we apply the tensorization of the KL-divergence over the product distribution induced by the dataset to get:
\begin{align}
D_{KL} \left( P_{\nu_1} \| P_{\nu_2 } \right) & = n_1 D_{KL} \left( \NN(\Delta, 1) \| \NN(\Delta, 1) \right)   \\
& \quad  + n_2 D_{KL} \left( \NN(0, 1)  \| \NN(2\Delta, 1) \right) \\
& = n_2 D_{KL} \left( \NN(0, 1)  \| \NN(2\Delta, 1) \right)
\end{align}
 For normal distribution, $D_{KL} \left( \NN(0, 1)  \| \NN(2\Delta, 1) \right) = 2 \Delta^2$. Therefore, choosing $\Delta = {1 \over 2 \sqrt n_2}$, we have
 \begin{align}
 \max_{i}\E_{\nu_i} \left[ Y(i) - Y(A)  \right] & \geq \frac{\Delta}{2} \left( 1 - \Delta \right)  \\
 & = {1 \over 8 \sqrt{n_2}}
 \end{align}

\end{proof}

\section{Proof of Theorem~\ref{thm::cov-est}}\label{appendix::cov-est}

\thmCovEst*

 \begin{proof}[Proof of Theorem~\ref{thm::cov-est}] 
 	
 	For all $k \in [M]$, let $\theta_{k, *}$ 
 	denote the solution to 
 	\begin{align*}
 	\min_{\theta \in \R^{d_k} } \sum_{ i \in [n]}  \left( \phi_{k}(x_i, a_i)^\top \theta  - f(x_i, a_i) \right)^2
 	\end{align*}
 	Throughout the proof, all expectations $\E$ denote $\E_X$ over $\DD$ and, within expectations over $X$, we will write $\hat k:= \hat k(X)$ for shorthand. From the definition of regret, we have
 	\begin{align*}
 	 \regret(\pi, \hat \pi)  &  = \E \left[  f(X, \pi(X)) - f(X, \hat \pi(X)) \right] \\
 	& \leq \E\left[  f(X, \pi(X)) - \<\phi_{\hat k}(X, \pi(X)), \theta_{\hat k,*} \>  \right]   + \E\left[  \<\phi_{\hat k}(X, \hat \pi(X)), \theta_{\hat k,*} \> - f(X, \hat \pi(X))   \right] \\ 
 	& \quad  +  \E \left[  \<\phi_{\hat k} (X, \pi(X)), \theta_{\hat k, *}\> - \< \phi_{\hat k}(X, \hat \pi(X)), \theta_{\hat k, *}\> \right]
 	\end{align*}
 	We now focus on bounding the contribution of the estimation error  to the regret. Using the condition in Lemma~\ref{lem::linear-concentration} to bound each $\| \hat \theta_k - \theta_{*, k}\|_{V_k}$, we have
 	\begin{align}
 	 & \E \left[   \<\phi_{\hat k} (X, \pi(X)), \theta_{\hat k, *}\> - \< \phi_{\hat k}(X, \hat \pi(X)), \theta_{\hat k, *}\> \right]  \\
 	& \leq \E \left[    \<\phi_{\hat k} (X, \pi(X)), \theta_{\hat k, *}\> - \< \phi_{\hat k}(X, \hat \pi(X)), \hat \theta_{\hat k}\> \right]  + \| \hat \theta_{\hat k} - \theta_{\hat k, *}\|_{V_{\hat k}} \cdot \E \| \phi_{\hat k}(X, \hat \pi(X)) \|_{V_{\hat k}^{-1}}  \\
 	& \leq \E \left[   \<\phi_{\hat k} (X, \pi(X)), \theta_{\hat k, *}\> - \< \phi_{\hat k}(X, \hat \pi(X)), \hat \theta_{\hat k}\> \right]  + \beta_{\lambda, \delta}(n, d_{\hat k}) \cdot \E  \| \phi_{\hat k}(X, \hat \pi(X)) \|_{V_{\hat k}^{-1}} 
 	\end{align}
 	Next, we apply the selection rule that determines $\hat k$ and $\hat \pi$ simultaneously, both of which are designed to maximize the penalized value estimate across actions and model classes. For any fixed $k \in [M]$, the previous display is bounded by
 	\begin{equation}\label{eq::cov-est-int1}
 	\begin{aligned}
 	 \E_X \left[   \<\phi_{\hat k} (X, \pi(X)), \theta_{\hat k, *}\> - \< \phi_{ k}(X,  \pi(X)), \hat \theta_{ k}\> \right]   + \beta_{\lambda, \delta}(n, d_k) \cdot \E  \| \phi_k(X,  \pi(X)) \|_{V_{ k}^{-1}}
 	\end{aligned}
 	\end{equation}
 	
 	There is now a potential mismatch between the predictions under $\theta_{*, \hat k}$ and the predictions under $\hat \theta_{k}$. To handle this, we will turn to the approximation error. For $k \in [M]$, define $\epsilon_k(X) := \abs{ f(X, \pi(X)) - \phi_k(X, \pi(X))^\top \theta_{k, *}}$.
 	
 	The first term in the previous display can be bounded using predictions under model class $k$ up to additive factors in $\epsilon_{\hat k}(X)$ and $\epsilon_k(X)$.
 	\begin{align*}
 & 	\<\phi_{\hat k} (X, \pi(X)), \theta_{\hat k, *}\> \\
 	 & \leq   \abs{\<\phi_{\hat k} (X, \pi(X)), \theta_{\hat k, *}\> - f(X, \pi(X))}   +  \abs{\<\phi_{ k} (X, \pi(X)), \theta_{ k, *}\> - f(X, \pi(X))}  + \<\phi_k(X, \pi(X)), \theta_{k, *} \> \\
 	& \leq \epsilon_{\hat k}(X) + \epsilon_k (X) + \<\phi_k(X, \pi(X)), \theta_{k, *} \>
 	\end{align*}
 	Then, conditioned on the same event from Lemma~\ref{lem::linear-concentration} and using the approximation error above, we may further bound (\ref{eq::cov-est-int1}) with
 	\begin{align*}
 	& \E \left[   \<\phi_{\hat k} (X, \pi(X)), \theta_{\hat k, *}\> - \< \phi_{ k}(X,  \pi(X)), \hat \theta_{ k}\> \right]   + \beta_{\lambda, \delta}(n, d_k) \cdot \E  \| \phi_k(X,  \pi(X)) \|_{V_{ k}^{-1}} \\
 	&\leq \E \left[   \<\phi_{k} (X, \pi(X)), \theta_{ k, *}\> - \< \phi_{ k}(X, \pi(X)), \hat \theta_k\> \right] +  \E  \left[  \epsilon_{\hat k}(X) + \epsilon_k (X)  \right] + \beta_{\lambda, \delta}(n, d_k) \cdot \E  \| \phi_k(X,  \pi(X)) \|_{V_{ k}^{-1}} \\ 
 	 & \leq  \E  \left[  \epsilon_{\hat k}(X) + \epsilon_k (X)  \right] + 2 \beta_{\lambda, \delta}(n, d_k) \cdot  \E \| \phi_k(X, \pi(X)) \|_{V_{k}^{-1}} 
 	\end{align*}
 	Applying this upper bound to the regret and using the fact that this holds for any fixed $k \in [M]$, we get
 	\begin{align}
 	& \regret(\pi, \hat \pi) \\ & \leq  \epsilon_{ k} (\pi, \hat \pi) + \E_X \left[  \epsilon_{\hat k}(X) \right]  
 	 + \min_{k \in [M]} \E_X \left\{ \epsilon_{k}(X) +2 \beta_{\lambda, \delta} (n, d_k) \cdot \| \phi_k(X, \pi(X)) \|_{V_{k}^{-1}} \right\} \\
 	 & \leq 2\sum_{k \in [M]} \epsilon_{ k} (\hat \pi, \pi)  +  \min_{k \in [M]} \left\{  2 \beta_{\lambda, \delta} (n, d_k) \cdot \E_X \| \phi_k(X, \pi(X)) \|_{V_{k}^{-1}} \right\}
 	\end{align}
 	Note here that $\hat k$ depends on $X$ and thus we cannot readily replace $\E_X  \left[ \epsilon_{\hat k}(X)\right]$ with $\max_{k'} \epsilon_{k'}(\pi, \hat \pi)$ in the first inequality. The sum over approximation errors is thus done to simplify the bound in terms of just $\epsilon_{k'}(\pi, \hat \pi)$ terms.
 	Finally, note that Lemma~\ref{lem::linear-concentration} establishes concentration of $\| \hat \theta_k - \theta_{*, k} \|_{V_k}$ for all $k \in [M]$ with probability at least $1 - M \delta$. Changing variables to $\delta' = M\delta$ proves the result.
 \end{proof}

\section{Proof of Theorem~\ref{thm::approx-est}}\label{appendix::slope}

\subsection{Single Model Guarantee}
We first begin with an independent result for a single $d$-dimensional $\FF$ that shows that one can bound the regret of a learned policy $\hat \pi$ by the approximation error $\tilde \epsilon$ of the optimal parameter $\bar \theta$ plus an estimation error term that depends on the complexity of the model class $d$. For clarity of notation, we drop dependence on the model class index $k$ in the subscript. Recall the definitions $\phi_i := \phi(x_i, a_i)$ and
\begin{align*}
    V & = {\lambda \over n} \I_d + {1 \over n}\sum_{i} \phi_i \phi_i^\top \\
    \hat \theta & = V^{-1} \left({1 \over n} \sum_i \phi_i y_i\right) \\
    \hat \pi(x) & \in \argmax_{a \in \AA } \< \hat \theta, \phi(x, a) \> \\
    \bar \theta & \in \argmin_{\theta \in \R^d}  \E_\mu \left( \< \phi(X, A), \theta\> - f(X, A)  \right)^2 \\
    \tilde \epsilon &= \min_{\theta \in \R^d}  2 \sqrt{ \CC(\mu)  \E_\mu \left( \< \phi(X, A), \theta\> - f(X, A) \right)^2 } 
\end{align*}

\begin{proposition} \label{prop::single}
For the above definitions, the following inequality holds with probability at least $1 - \delta$:
\begin{align*}
    \| \hat \theta - \bar \theta \|_V & \leq \sqrt{ \frac{\lambda \| \bar \theta  \|^2 }{n}} +  C_4 \sqrt{ { d \over n} } \| V^{-1/2} \| \cdot \log(4d/\delta) + \sqrt{\frac{C_1 d + C_2 \sqrt{d \log(4d/\delta) } + C_3  \log(4d/\delta)}{n}}
\end{align*}
Furthermore, under the same event, the regret $\regret(\pi, \hat \pi)$ is bounded above by:
\begin{align*}
\textstyle
     \tilde \epsilon + \left( \sqrt{ \frac{\lambda \| \bar \theta  \|^2 }{n}} +  C_4 \sqrt{ { d \over n} } \| V^{-1/2} \| \cdot \log(4d/\delta) + \sqrt{\frac{C_1 d + C_2 \sqrt{d \log(4d/\delta) } + C_3  \log(4d/\delta)}{n}} \right)  \cdot \E_X \max_{a}  \| \phi(X, a) \|_{V^{-1}}
\end{align*}
where $C_4 = 192$ and $C_{1:3}$ are defined in Lemma~\ref{lem::linear-concentration}.
\end{proposition}
\begin{proof}
The regret decomposes as
\begin{align*}
    \regret(\pi, \hat \pi) & = \E_X \left[ f(X, \pi(X)) - f(X, \hat \pi(X))\right] \\
    & = \E_X \left[ \left( f(X, \pi(X)) - \< \phi(X, \pi(X)), \bar \theta\> \right) + \left( \< \phi(X, \hat \pi(X)), \bar \theta\> - f(X, \hat \pi(X)) \right) \right] \\
    & \quad + \E_X \left[\<\phi(X, \pi(X)) - \phi(X, \hat \pi(X)), \bar \theta \>\right]
\end{align*}
By Jensen's inequality and Definition~\ref{def::concentrability}, the first expectation can be bounded as
\begin{align*}
   & \E_X \left[ \left( f(X, \pi(X)) - \< \phi(X, \pi(X)), \bar \theta\> \right) + \left( \< \phi(X, \hat \pi(X)), \bar \theta\> - f(X, \hat \pi(X)) \right) \right]  \\
   & \leq 2 \sqrt{ \CC(\mu) \E_\mu \left( \< \phi(X, A), \bar \theta\> - f(X, A) \right)^2 } \\
    & = \tilde \epsilon
\end{align*}
For the second expectation, it remains to show that the policy $\hat \pi$ selects actions nearly as well as $\pi$ with respect to $\bar \theta$. For convenience, define $\phi(\pi) := \E_X \phi(X, \pi(X))$ for features $\phi$ and policy $\pi$. Then,
\begin{align*}
    \<\phi(\pi), \bar \theta\> - \< \phi(\hat \pi), \bar \theta\>  & = \< \phi(\pi), \bar \theta\> - \< \phi(\hat \pi), \hat  \theta\> + \< \phi(\hat \pi), \hat  \theta\> - \< \phi(\hat \pi), \bar \theta\> \\
    & \leq \< \phi(\pi), \bar \theta\> - \< \phi( \pi), \hat  \theta\> + \< \phi(\hat \pi), \hat  \theta\> - \< \phi(\hat \pi), \bar \theta\> \\
    & \leq \| \bar \theta - \hat \theta\|_{V} \cdot \E_X \| \phi(X, \pi(X)) \|_{V^{-1}} + \| \bar \theta - \hat \theta\|_{V} \cdot \E_X\| \phi(X, \hat\pi(X)) \|_{V^{-1}} \\
    & \leq 2 \| \bar \theta - \hat \theta \|_{V} \cdot \E_X \max_{a} \| \phi(X, a) \|_{V^{-1}}
\end{align*}
where the first inequality uses the fact that $\hat \pi$ selects actions to maximize the reward predicted with $\hat \theta$, the second inequality applies Cauchy-Schwarz and the last inequality takes the worst-case action.
Thus, we focus on the concentration of $\| \hat \theta - \bar \theta\|_{V}$. We use the previous definitions of $\phi_i = \phi(x_i, a_i)$, $f_i := f(x_i, a_i)$, and $\eta_i :=  \eta_i(a_i)$. Furthermore, we define the error term $e_i = f_i - \phi_i^\top \bar\theta$.  
\begin{align*}
    \| \hat \theta - \bar \theta \|_{V}  & =\norm{ \frac{1}{n }V^{-1} \sum_{i} \phi_i y_i - \bar \theta }_V \\
    & = \norm{ \frac{1}{n }V^{-1} \sum_{i} \phi_i \left(\phi_i^\top \bar \theta + e_i + \eta_i \right)  - \bar \theta }_V \\
    & = \norm{ \frac{1}{n } V^{-1} \sum_{i} \phi_i \eta_i + \frac{1}{n }V^{-1} \sum_i \phi_i  e_i   - \lambda V^{-1}\bar \theta }_V  \\
    & \leq {1 \over n} \norm{  \lambda V^{-1} \bar \theta  }_{V} + \norm{ \frac{1}{n }V^{-1} \sum_i \phi_i \eta_i }_V + \norm{ \frac{1}{n }V^{-1} \sum_i \phi_i e_i }_V  
\end{align*}
The first term is bounded above as ${1 \over n} \|  \lambda V^{-1} \bar \theta  \|_{V} \leq  \sqrt{{\lambda \| \bar \theta \|^2 \over n } }$. For the second term, we appeal to Lemma~\ref{lem::hanson-wright-application} to show that
\begin{align*}
    \norm{ \frac{1}{n}V^{-1} \sum_i \phi_i \eta_i }_V^2 & = \norm{ \frac{1}{n} V^{-1/2} \sum_i \phi_i \eta_i }^2 \\
    & \leq \frac{C_1 d + C_2 \sqrt{d \log(1/\delta) } + C_3  \log(1/\delta)}{n}
\end{align*}
conditional on $x_{1:n}$ and $a_{1:n}$ for constants $C_1, C_2, C_3 > 0$ defined in Lemma~\ref{lem::hanson-wright-application}
with probability at least $1 - \delta$. For the third term, we note that the expectation inside the norm is zero and use Lemma~\ref{lem::errors-concentration} to show concentration, yielding:
\begin{align*}
    \norm{ \frac{1}{n }V^{-1} \sum_i \phi_i e_i }_V  & \leq 64 (1 + 2\|\bar \theta \| ) \sqrt{d/n} \| V^{-1/2} \| \cdot \log(2d/\delta) \\
    & \leq C_4\sqrt{d/n} \| V^{-1/2} \| \cdot \log(2d/\delta)
\end{align*}
with probability at least $1 - \delta$ for a constant $C_4 = 192$ since $\| \bar \theta \| \leq 1$ by assumption. 

Therefore, by the union bound, we are able to conclude that
\begin{align*}
    \| \hat \theta - \bar \theta \|_V & \leq \sqrt{ \frac{\lambda \| \bar \theta  \|^2 }{n}} +  C_4 \sqrt{ { d \over n} } \| V^{-1/2} \| \cdot \log(4d/\delta) + \sqrt{\frac{C_1 d + C_2 \sqrt{d \log(4d/\delta) } + C_3  \log(4d/\delta)}{n}}
\end{align*}
with probability at least $1 - \delta$ for $C_4 = 192$ and $C_{1:3}$ are defined in Lemma~\ref{lem::linear-concentration}.
\end{proof}

Proposition~\ref{prop::single} ensures the validity of choosing 
\begin{align*}\zeta_k(\delta) =  \sqrt{ \frac{\lambda }{n}} +  192 \sqrt{ { d_k \over n} } \| V^{-1/2}_k \| \cdot \log(4d_k/\delta) + \sqrt{\frac{5 d_k + 10 \sqrt{d_k \log(4d_k/\delta) } + 10  \log(4d_k/\delta)}{n}}
\end{align*}
Note that we have used the assumption that $\| \bar \theta_k \| \leq 1$.

\subsection{An Improved Analysis of the SLOPE Estimator}
In order to simplify notation, we will denote $f(\pi) = \E_X f(X, \pi(X))$ and $\phi(\pi) = \E_X \phi(X, \pi(X))$ for deterministic policies $\pi$ and features $\phi$.

Algorithm~\ref{alg::approx-est} relies on the validity of a version of the SLOPE estimator introduced by \cite{su2020adaptive}. Recall that we construct the following value estimators:
\begin{align*}
    \hat v_k(\pi) = \E_X \< \phi(X, \pi(X)), \hat \theta_k \>
\end{align*}
Proposition~\ref{prop::single} ensures that they satisfy the following guarantee.
\begin{lemma}
    Let the event of Proposition~\ref{prop::single} hold  for all model classes $k \in [M]$. Then, for any $k$ and policy $\pi$,
    \begin{align*}
        | \hat v_k(\pi) - f(\pi)  | \leq \tilde \epsilon_k + \zeta_k(\delta) \cdot \E_X \max_a \| \phi(X, a) \|_{V_k^{-1}}
    \end{align*}
\end{lemma}
\begin{proof}
     The event ensures that $\| \hat \theta_k - \bar \theta_k\|  \leq \zeta_k(\delta)$. This holds for all model classes with probability at least $1 - M\delta$. Therefore, we have 
    \begin{align*}
         \hat v_k(\pi) - f(\pi) & = \< \phi_k(\pi), \hat \theta_k\> -  f(\pi) \\
         & = \< \phi_k(\pi), \hat \theta_k\> -\< \phi_k(\pi), \bar \theta_k\> +  \< \phi_k(\pi), \bar \theta_k\> -  f(\pi) \\
         & \leq \< \phi_k(\pi), \hat \theta_k\> -\< \phi_k(\pi), \bar \theta_k\> + \sqrt{ \CC(\mu) \E_\mu \left(\< \phi_k(X, A), \bar \theta_k \> - f(X,A) \right)^2 } \\
         & \leq \tilde \epsilon_k + \| \hat \theta_k - \bar \theta_k \|_{V_k} \cdot \E_X\max_a \|\phi_k(X, a) \|_{V_k^{-1}}  \\
         & \leq \tilde \epsilon_k + \zeta_k(\delta) \cdot \E_X\max_a \|\phi_k(X, a) \|_{V_k^{-1}} 
    \end{align*}
    where we have again used Jensen's inequality, concentrability in Definition~\ref{def::concentrability}, and Cauchy-Schwarz.
\end{proof}

Next, we verify that $\tilde \epsilon_k$ is decreasing in $k$ while $\zeta_k(\delta) \cdot \E_X \max_a \| \phi(X, a) \|_{V_k^{-1}}$. 

\begin{lemma}\label{lem::ordering}
	The following conditions hold for all $k \in [M - 1]$:
	\begin{enumerate}
		\item $\tilde \epsilon_k \geq \tilde \epsilon_{k + 1}$  and
		\item $\zeta_k(\delta) \cdot \E_X \max_a \| \phi_k(X, a) \|_{V_k^{-1}} \leq \zeta_{k + 1}(\delta) \cdot \E_X \max_a \| \phi_{k + 1}(X, a) \|_{V_{k + 1}^{-1}}$.
	\end{enumerate}
\end{lemma}
\begin{proof}
	The first condition is trivially true since for any $\theta \in \R^{d_k}$ we have $\theta' \in \R^{d_{ k + 1} }$, which equals $\theta$ in the top $d_k$ coordinates and is zero in the bottom $d_{k + 1} - d_k$ coordinates since the model classes are nested. This at least achieves the excess risk of $\theta$ and therefore $\tilde \epsilon_k \geq \tilde \epsilon_{k + 1}$.
	
	For the second condition, observe that one immediately has $\xi_k(\delta)  \leq \xi_{k + 1}(\delta)$. It suffices to show that the second factor is also increasing.  Lemma~\ref{lem::schur} shows that in general for nested vectors and positive definite matrices:
	\begin{align}
	\| \phi_k(X, a) \|_{V_k^{-1}} \leq \| \phi_{k + 1}(X, a)  \|_{V_{k + 1}^{-1}},
	\end{align}
	proving the claim.
\end{proof}

We are now ready to prove that the SLOPE estimator from Algorithm~\ref{alg::approx-est} is adaptive. Note that this proof is done in general and may be of independent interest as it requires fewer assumptions than that of \cite{su2020adaptive}. Consider estimators $\hat v_1, \ldots, \hat v_M$ of a quantity $v \in \R$ and define 
\begin{align*}
    \hat{k} = \min \{ k : |\hat{v}_k - \hat{v}_\ell| \leq 2 \xi_k ,  \quad \forall \ell > k \}
\end{align*}

\begin{theorem}\label{thm::slope}
Let $\hat v_1, \ldots, \hat v_M$ be estimators of a quantity $v \in \R$ with parameters $(\psi_k)_{k \in [M]}$ and $(\xi_k)_{k \in [M]}$ satisfying
\begin{enumerate}
    \item $| \hat v_k - v| \leq \psi_k + \xi_k$ for all $k \in [M]$
    \item $\psi_{k} \geq \psi_{k + 1}$ for all $ k\in [M - 1]$
    \item $\xi_k \leq \xi_{k + 1}$ for all $k \in [M - 1]$
\end{enumerate}
Then, the estimator $\hat v$ defined above satisfies
\begin{align*}
    | \hat v - v| \leq C \min_{k} \left\{ \psi_k + \xi_k \right\}
\end{align*}
where $C = 5$.
\end{theorem}
\begin{proof}
Let $k_* =  \argmin_{k \in [M]} \left\{ \psi_k + \xi_k \right\}$.
    To prove the claim, we handle to cases: (1) $\hat k < k_*$ and (2) $\hat k > k_*$. Otherwise, the selection is already correct.
    In the first case, we have that $\hat k$ intersects all intervals above it including $k_*$. Therefore
    \begin{align*}
        | v - \hat v_{\hat k} | & \leq |\hat v_{\hat k} - \hat v_{k_*}| + |v - \hat v_{k_*}|  \\
        & \leq 2 \xi_{\hat k} + 2\xi_{k_*} + \psi_{k_*} + \xi_{k_*} \\
        & \leq 5\left( \psi_{k_*} + \xi_{k_*}\right) 
    \end{align*}
    For the second case, we have that $i = \hat k - 1$, which satisfies $i_* \geq i$ does intersect with some $j \in [\hat k, M]$. Therefore
    \begin{align*}
        2\xi_i + 2\xi_j \leq | \hat v_i - \hat v_j | \leq \psi_i + \xi_i + \psi_j + \xi_j
    \end{align*}
    by definition. It follows then that
    \begin{align*}
        \xi_i + \xi_j \leq \psi_i  + \psi_j \leq 2\psi_{k_*}
    \end{align*}
    since $k_* \leq i, j$. Therefore,
    \begin{align*}
        | v - \hat v_{\hat k} | & \leq \psi_{\hat k} + \xi _{\hat k} \\
        & \leq \psi_{k_*} + \xi _{j} \\
        & \leq \psi_{k_*} + \xi _{j} + \xi_i \\
        & \leq 3 \psi_{k_*} \\
        & \leq 3 \left( \psi_{k_*} + \xi_{k_*}\right) 
    \end{align*}
    Since all cases have been handled, we see that the claim is satisfied with $C = 5$.
\end{proof}

\subsection{Proof of Model Selection Guarantee}
Equipped with the single model class guarantees and the SLOPE estimator guarantees, we are ready to prove the final result, which is restated below.

\thmApproxEst*

\begin{proof}[Proof of Theorem~\ref{thm::approx-est}]

Recall that Proposition~\ref{prop::single} guarantees that $\|\hat \theta_k - \bar \theta \|_V \leq \zeta_k(\delta)$ for all $k$ with probability at least $1 - M\delta$.  From here on, assume this event holds. In order to derive the results, we first note that Lemma~\ref{lem::ordering} ensures the ordering properties of $\epsilon_k$ and $\zeta_k(\delta) \cdot \E \max \| \phi_k(X,a) \|_{V_{k}^{-1}}$. Therefore, it is valid to apply Theorem~\ref{thm::slope} with $\psi_k = \tilde \epsilon_k$ and $\xi_k = \zeta_k \cdot \E \max \| \phi_k(X,a) \|_{V_{k}^{-1}}$.

Note that the selection rule ensures that $\hat \pi \equiv \hat \pi_{\hat \ell}$ where $\hat \ell = \argmax_{\ell} \hat v(\hat \pi_\ell)$. 
	From the regret decomposition, we have, for any fixed $\ell \in [M]$,
	\begin{align*}
	\regret(\pi, \hat \pi) & = f(\pi) - f(\hat \pi) \\
	& =  f(\pi)   - \hat v(\hat \pi)  +  \hat v(\hat \pi)  - f(\hat \pi) \\
	& \leq  f(\pi)  - \hat v(\hat \pi_{\hat \ell}) + C' \min_{k} \left\{   \tilde \epsilon_k + \zeta_k(\delta) \cdot \E \max \| \phi_k(X,a) \|_{V_{k}^{-1}} \right\} \\
& \leq f(\pi)  - \hat v(\hat \pi_{\ell}) + C' \min_{k} \left\{   \tilde \epsilon_k + \zeta_k(\delta) \cdot \E \max \| \phi_k(X,a) \|_{V_{k}^{-1}} \right\} \\
& = f(\pi)  - f(\hat \pi_\ell) + f(\hat \pi_\ell) -  \hat v(\hat \pi_{\ell}) + C' \min_{k} \left\{   \tilde \epsilon_k + \zeta_k(\delta) \cdot \E \max \| \phi_k(X,a) \|_{V_{k}^{-1}} \right\} \\
& \leq f(\pi)  - f(\hat \pi_\ell)  + 2C' \min_{k} \left\{   \tilde \epsilon_k + \zeta_k(\delta) \cdot \E \max \| \phi_k(X,a) \|_{V_{k}^{-1}} \right\} \\
& = f(\pi)  - f(\hat \pi_\ell)  + 2C' \min_{k} \left\{   \tilde \epsilon_k + \zeta_k(\delta) \cdot \E \max \| \phi_k(X,a) \|_{V_{k}^{-1}} \right\} \\
& \leq \tilde \epsilon_\ell  + 2\zeta_\ell(\delta)\cdot \E \max \| \phi_\ell(X,a) \|_{V_{\ell}^{-1}}   + 2C' \min_{k} \left\{   \tilde \epsilon_k + \zeta_k(\delta) \cdot \E \max \| \phi_k(X,a) \|_{V_{k}^{-1}} \right\}
	\end{align*}
where the first  and third inequalities follow from Theorem~\ref{thm::slope} with the constant $C' = 5$ and the second uses the selection rule of $\hat \ell$. The last inequality uses Proposition~\ref{prop::single}. Therefore,
\begin{align*}
    \regret(\pi, \hat \pi) & \leq (2C + 1) \tilde \epsilon_\ell +  2(C + 1)  \zeta_\ell(\delta)\cdot \E \max \| \phi_\ell(X,a) \|_{V_{\ell}^{-1}} 
\end{align*}
Recall that $\ell \in [M]$ was arbitrary and the assumed event occurs with probability at least $1 - M \delta$. The proof of the claim is completed by a change of variables $\delta' = M \delta$. Therefore, the claim is satisfied by choosing $C = 12$.

\end{proof}

\subsection{Technical Lemmas}

\begin{lemma}\label{lem::errors-concentration}
     With probability at least $1 - \delta$, 
    \begin{align}
        \norm{ V^{-1/2} \sum_i \phi_i e_i } \leq C ( 1 + 2\|\bar \theta\|)  \sqrt{n d} \| V^{-1/2} \| \cdot \log(2d/\delta)  
    \end{align}
    where $C = 64$
\end{lemma}
\begin{proof}
    Note that we have $\E_\mu \left[ \phi_i e_i\right] = \E_\mu \left[ \phi_i ( f_i  -\phi_i^\top \bar \theta) \right] = \E_\mu \phi(X, A) f(X, A) - \Sigma \bar \theta$ which equals zero by first order optimality conditions applied to the minimizer $\bar \theta$. Therefore it suffices to show concentration of $\sum_i \Sigma^{-1/2} \phi_i e_i$ around its mean. 
    
    Define $\tilde \phi_i = \Sigma^{-1/2} \phi_i$ and define $\tilde \phi_i^j$ as the $j$th coordinate of the sample $\tilde \phi_i$. 
        
    Note that $Z_i^j := \tilde \phi_i^j\left( f_i -  \tilde \phi_i^\top  \Sigma^{1/2}\bar \theta \right)$ is sub-exponential with parameter $\|Z_i^j \|_{\psi_1} \leq C_1 + C_2\| \Sigma^{1/2} \bar \theta \| \leq C_1 + 2C_2 \| \bar \theta \|$ where $C_1 = 1$ and $C_2 = 1$ $C_1, C_2 > 0$ by Lemma~\ref{lem::sube-product} and Lemma~\ref{lem::moment-bound}. This follows because $\| \tilde \phi_i \|_{\psi_2} \leq 1$ and $f_i \in [-1, 1]$ by assumption. Therefore by Bernstein's inequality and multiplying by $\Sigma^{-1/2}$,
    \begin{align*}
        | \sum_{i} \phi_i^j e_i | & \leq 32 \sqrt{ (1 + 2\| \bar \theta \| )^2 n \cdot  \log(2/\delta) }  + 32 (1 + 2\| \bar \theta \| ) \log(2/\delta)   \\
        & \leq 64 (1 + 2\| \bar \theta \| ) \sqrt{ n  } \cdot \log(2/\delta) 
    \end{align*}
    with probability at least $1 - \delta$. We have used the fact that $\delta \leq 1/e$ so that $\sqrt{\log(2/\delta)} \leq \log(2/\delta)$. Taking the union bound over all coordinates $j \in [d]$ and applying standard norm inequalities, we have
    \begin{align*}
        \| \sum_{i} \phi_i e_i \| \leq  64(1 + 2\| \bar \theta \| ) \sqrt{ n d } \cdot \log(2/\delta) 
    \end{align*}
    with probability at least $1 - d\delta$ by the union bound. The result then follows by change of variables with $\delta' = d \delta$ and applying the $\ell_2$ matrix norm inequality.
\end{proof}

\section{Proof of Theorem~\ref{thm::hold-out}}\label{appendix::hold-out}

We define the expected regression loss as $L_k(\theta) = \E \hat L_k(\theta)$ for $\theta \in \R^{d_k}$. We first require the following concentration result which is immediate from the selection rule via Bernstein's inequality.
\begin{lemma}\label{lem::hold-out-concentration}
    There is a constant $C > 0$ such that with probability at least $1 - \delta$, \begin{align*}| \hat L_k(\hat \theta_k) - L_k(\hat \theta_k)|  \leq C \sqrt{ \frac{ ( 1+ \|\hat \theta_k\|)^4 }{n_{out}}} \cdot \log(2M/\delta)\end{align*} for all $k \in [M]$.
\end{lemma}

\begin{proof}
    Define $Z_{k, i} = \< \phi_{k, i}, \hat \theta_k\> - y_i$. Note that $\|Z_{k, i} \|_{\psi_2} \leq 2\| \hat \theta_k\| + 2$ since $ \| \Sigma^{-1/2}\phi_{k, i} \|_{\psi_2} \leq 1$ and $\|y_i \|_{\psi_2}  \leq 2$. Therefore $\| Z_{k, i}^2 \|_{\psi_1} \leq (2 + 2\| \hat \theta_k\| )^2$. By Bernstein's inequality, we have
    \begin{align*}
        | \hat L_k(\hat \theta_k) - L_k(\hat \theta_k)|  & = | {1 \over n_{out}} \sum_{i} Z_{k, i}^2 - \E [Z_{k, i}^2] |   \\
        & \leq 32 \sqrt{ \frac{4(2 +2 \| \hat \theta_k\| )^4\log(2M/\delta)}{n_{out} } } + \frac{64 (2  + 2\| \hat \theta_k\| )^2  \log(2M/\delta) }{n_{out}}
    \end{align*}
    with probability at least $1 - \delta$ for all $k$. It is assumed that $\delta \leq 1/e$. Therefore, ${1 \over n_{out}} \leq {1 \over \sqrt{n_{out}}}$ and $\sqrt{\log(M/\delta)} \leq \log(M/\delta)$. Applying these two upper bounds to the terms above and then summing them gives the result.
    
\end{proof}

Armed with this result, we turn to the proof of Theorem~\ref{thm::hold-out}, restated below.

\thmHoldOut*

\begin{proof}[Proof of Theorem~\ref{thm::hold-out}]
Recall that $Y(A) = f(X, A) + \eta(A)$ is the observed random reward taking action $A$ where $\eta$ is the noise vector. By linearity of expectation, for any $\hat k \in [M]$,
\begin{align*}
    L_{\hat k} (\theta) & = \E_\mu \left( \< \phi_{\hat k}(X, A), \theta\> - Y(A) \right)^2 \\
    & = \E_\mu \left( \< \phi_{\hat k}(X, A), \theta\> - f(X, A) + f(X, A) -  Y(A) \right)^2 \\
    & = \E_\mu \left( \< \phi_{\hat k}(X, A), \theta\> - f(X, A) \right)^2 + \E_\mu \left( f(X, A) -  Y(A) \right)^2  \\
    & \quad + 2 \E_\mu \left( \< \phi_{\hat k}(X, A), \theta\> - f(X, A) \right) \left(f(X, A) -  Y(A)\right) \\
    & = \E_\mu \left( \< \phi_{\hat k}(X, A), \theta\> - f(X, A) \right)^2 + \E_\mu \left( f(X, A) -  Y(A) \right)^2 
\end{align*}
Applying the tower rule of the expectation to the last term conditioned on $(X, A)$ yields the above results since $\eta$ is zero-mean independent noise.
Then,
    \begin{align}
        \E_\mu \left(\<\phi_{\hat k} (X, A), \hat \theta_{\hat k} \> - f(X, A) \right)^2  
        & = L_{\hat k}(\hat \theta_{\hat k}) - \E_\mu \left(f(X, A) - Y(A) \right)^2 \\
        & \leq \hat L_{\hat k} (\hat \theta_{\hat k}) - \E_\mu \left(f(X, A) - Y(A) \right)^2 +C \sqrt{ \frac{ ( 1+ \|\hat \theta_{\hat k}\|)^4 }{n_{out}}} \cdot \log(M/\delta) \\
        & \leq \hat L_k(\hat \theta_{k}) - \E_\mu \left(f(X, A) - Y(A) \right)^2 + C \sqrt{ \frac{ ( 1+ \max_\ell \|\hat \theta_\ell\|)^4 }{n_{out}}} \cdot \log(M/\delta) \\
        & \leq  L_k(\hat \theta_{ k})  - \E_\mu \left(f(X, A) - Y(A) \right)^2 + 2C \sqrt{ \frac{ ( 1+ \max_\ell \|\hat \theta_{\ell} \|)^4 }{n_{out}} } \cdot \log(M/\delta)\\
        & \leq \E_\mu \left(\<\phi_{ k} (X, A), \hat \theta_{ k} \> - f(X, A) \right)^2 + 2C \sqrt{ \frac{ ( 1+ \max_\ell \|\hat \theta_{\ell} \|)^4 }{n_{out}} } \cdot \log(M/\delta)
    \end{align}
    with probability at least $1 - \delta$. The first inequality follows from Lemma~\ref{lem::hold-out-concentration}. The second inequality follows from the choice of $\hat k$ to minimize $\hat L_k(\hat \theta_k)$.
    
    Therefore, we can apply the following regret bound for any $k\in[M]$:
    \begin{align*}
        \regret(\pi, \hat \pi) &  = f(\pi) - f(\hat \pi)  \\
        & \leq f(\pi) - \<\phi(\hat \pi), \hat \theta_{\hat k}  \> + \<\phi(\hat \pi), \hat \theta_{\hat k}  \> -  f(\hat \pi) \\
        & \leq f(\pi) - \<\phi( \pi), \hat \theta_{\hat k}  \> + \<\phi(\hat \pi), \hat \theta_{\hat k}  \> -  f(\hat \pi) \\
        & \leq 2\sqrt{  \CC(\mu) \E_\mu \left(\<\phi_{\hat k} (X, A), \hat \theta_{\hat k} \> - f(X, A) \right)^2 } \\ 
        & \leq 2\sqrt{  \CC(\mu) \E_\mu \left(\<\phi_{ k} (X, A), \hat \theta_{ k} \> - f(X, A) \right)^2 +2C \CC(\mu) \sqrt{ \frac{ ( 1+ \max_\ell \|\hat \theta_{\ell} \|)^4 }{n_{out}} } \cdot \log(M/\delta) } \\
     & \leq     2\sqrt{  \CC(\mu) \E_\mu \left(\<\phi_k (X, A), \bar \theta_k\> - f(X, A) \right)^2 } + 2\sqrt{\CC(\mu)} \| \hat \theta_k - \bar \theta_k \|_{\Sigma_k}  \\
     & \quad + 2\sqrt{\CC(\mu)}\left( \frac{C_1 \left(1 + \max_{\ell} \| \hat \theta_\ell\| \right)^4 \log^2 (M/\delta) }{n_{out}}\right)^{1/4} \\
     & = \tilde \epsilon_k + 2\sqrt{\CC(\mu)} \| \hat \theta_k - \bar \theta_k \|_{\Sigma_k} + 2\sqrt{\CC(\mu)}\left( \frac{C_1 \left(1 + \max_{\ell} \| \hat \theta_\ell\| \right)^4 \log^2 (M/\delta) }{n_{out}}\right)^{1/4}
    \end{align*}
    Note that the third inequality follows from applying Jensen's inequality and Definition~\ref{def::concentrability}. The fourth inequality applies the previous display.
    
\end{proof}

\section{Supporting Lemmas}
In this section, we state several independent results that support the proofs of the main results.

\subsection{Nestedness Properties}
Let $M\in \R^{d \times d}$ be a positive definite matrix of the form
\begin{align}
M = \begin{bmatrix}
A  & B \\ B^\top &  D
\end{bmatrix}
\end{align}
where $A \in \R^{d_1 \times d_1}$ is also a positive definite matrix and $D \in \R^{d_2 \times d_2}$ and $B \in \R^{d_1 \times d_2}$.
\begin{lemma}\label{lem::schur}
	Let $a \in \R^{d_1}$ and $b \in \R^{d_2}$ be arbitrary. The following inequality holds:
	\begin{align}
	\begin{bmatrix}
	a \\ b
	\end{bmatrix}^\top 
	M^{-1} \begin{bmatrix}
	a \\ b
	\end{bmatrix} \geq a^\top A^{-1} a
	\end{align}
\end{lemma}
\begin{proof}
	By Schur complement inverse rules:
	\begin{align}
	\begin{bmatrix}
	a \\ b
	\end{bmatrix}^\top 
	M^{-1} \begin{bmatrix}
	a \\ b
	\end{bmatrix} & =  a^\top A^{-1} a + a^\top A^{-1} B (M / A)^{-1} B^\top A^{-1} a - 2b^\top (M / A)^{-1} B^\top A^{-1} a + b^\top (M / A)^{-1} b \\
	& \geq  a^\top A^{-1} a + a^\top A^{-1} B (M / A)^{-1} B^\top A^{-1} a - a^\top A^{-1} B ( M / A)^{-1} B^\top A^{-1} a \\
	& = a^\top A^{-1} a
	\end{align}
	where the inequality follows from optimizing over $b$.
\end{proof}

\subsection{Concentration of Quadratic Forms}
The following is a restatement of Lemma 14 of \cite{hsu2012random} for convenience, which can be interpreted as a version of the Hanson-Wright inequality \citep{rudelson2013hanson}.
\begin{lemma}\label{lem::hanson-wright}
	Let $A \in \R^{m \times n}$ be a matrix and $\Sigma = AA^\top$. Let $X \in \R^n$ be a random vector with independent coordinates $X_1, \ldots, X_n$ such that $\E X_i = 0$ and $\E \exp (\lambda X_i) \leq \E \exp \left( {\lambda^2 \sigma^2 \over 2 } \right)$ for all $\lambda^2$. Then,
	\begin{align}
	P\left(  \| A X \|^2 \leq  \sigma^2\tr \Sigma + 2 \sigma^2 \sqrt{ \tr(\Sigma^2) \log(1/\delta) } + 2 \sigma^2\| \Sigma \| \log(1/\delta) \right) \leq \delta
	\end{align}
	
\end{lemma}

 \section{Additional Experiment Details}
 
 In this section, we provide some additional details regarding the experimental results presented in Section~\ref{sec::exp}. We start with details that are common to both of the settings considered.
 In order to evaluate the performance of algorithms, within each trial, we generated a test set of $n_{\text{test}} = 500$ samples. All algorithms were thus compared on the same data within a trial. For both the batch dataset and the test set, noise was artificially generated on rewards by sampling from a standard normal distribution $\NN(0, 1)$ such that $\eta(a) \sim \subg(1)$ for all $a \in \AA$. Regret was computed by taking the difference between the optimal policy $\pi_*$ and the learned policy evaluated on the same test set.  Thus, the points approximately (up to noise) represent $\regret(\pi_*, \hat \pi_n)$ where $\hat \pi_n$ is the learned policy after $n$ batch samples.
 
 The data collection policy was generated as a policy independent of the observed state. Thus $\mu(a | x)  = \mu(a' | x)$ for all $a, a', x$. We generated the probabilities of sampling arms by sampling from a standard Dirichlet distribution of $|\AA|$ values. For the algorithms, penalization terms (i.e. the estimation error) typically depends on constants being chosen sufficiently large to ensure a confidence interval is valid. However, choosing large values in practice can lead to unnecessarily poor convergence. We found that multiplying by $C = 0.1$ yielded good performance in most settings.
 
 \subsection{Complexity-Coverage Setting}
 
 In this section, all random quantities were generated by sampling multivariate normal distributions.
 We first generated a random vector $(f(x, a))_{x \in \XX, a \in \AA}$, which specifies the average reward for each state-action pair. In order to generate a set of linear models (feature maps) that all satisfy realizability, we began with an input $d_{hid}$ and randomly generated $d_{hid} -1$ vectors $v_1, \ldots, v_{d_{hid} - 1}$ of length $|\XX|| \AA|$ and solved for the last $v_{d_{hid}}$ by subtracting these off the reward. This ensures a particular linear combination of the $v$s equals the vector $(f(x, a))_{x \in \XX, a \in \AA}$. This procedure was repeated for various values of $d_{hid}$ and the resulting feature vectors were scaled up to $d = |\XX| |\AA|$ by multiplying by a random matrix $A$ with elements generated from $\NN(0, 1)$. This ensures that the feature maps are not simply equivalent linear transformations of each other.
 
 \subsection{Approximation-Complexity Setting}
 
 In contrast the previous setting, we considered an infinite state space where, for each action, a $d = 100$ dimensional covariate vector is sampled from a multivariate normal distribution with mean $0$ and covariance matrix $\Sigma_a$, where $\Sigma_a$ was also randomly generated. As mentioned in the main text, we constructed model classes by truncating the original covariate vector to small dimensions, thus inducing a nested structure. Since $d_* = 30$, some of these choices result in misspecified models. For the SLOPE method, in order to estimated the predicted values to generate each $\hat v_k$, we used a validation set of unlabeled samples (i.e. no revealed reward).

\end{document}